\pdfoutput=1

\documentclass[11pt]{article}
\usepackage{bbm}
\usepackage{fix-cm}
\usepackage{acl}

\usepackage{tikz}

\usepackage{enumitem}

\usepackage{xfrac}
 
\usepackage{mathtools}
\mathtoolsset{showonlyrefs,showmanualtags}
\usepackage{amsfonts}
\usepackage{bm}
\usepackage{times}
\usepackage{latexsym}
\usepackage{amsfonts,amsmath,amsthm, amssymb}
\usepackage{graphicx}
\usepackage{placeins}
\usepackage{array}
\usepackage{ifthen}
\usepackage{tabularx}
\usepackage{booktabs}
\usepackage{xcolor}
\usepackage{relsize}
\usepackage{stfloats}
\usepackage{float}
\usepackage{multirow}
\usepackage{makecell}
\usepackage{graphicx}
\usepackage{subcaption}
\usepackage{makecell}

\newcommand{\rv}[1]{\mathbf{#1}}
\newcommand{\myvec}[1]{\mathbf{#1}}
\newcommand{\mymat}[1]{\bm{#1}}

\usepackage[T1]{fontenc}

\definecolor{cbBlue}{HTML}{0072B2}
\definecolor{cbOrange}{HTML}{E69F00}
\definecolor{cbGreen}{HTML}{009E73}
\definecolor{cbPurple}{HTML}{CC79A7}

\newcommand{\rainbowicon}[1][cbBlue]{%
  \begin{tikzpicture}[
    x=0.5em,y=0.65em,
    baseline=-0.1em,
    draw=#1,line width=0.055em,line cap=round
  ]
    \foreach \r in {0.35,0.60,0.85}
      \draw (\r,0) arc[start angle=0,end angle=180,radius=\r];
  \end{tikzpicture}%
}
\usepackage{microtype}

\usepackage{listings}
\usepackage{xcolor}
\usepackage{mdframed}

\lstdefinelanguage{SQL}{
  morekeywords={SELECT,FROM,WHERE,JOIN,ON,AND,OR,NOT,NULL,IS,AS,ORDER,BY,DESC,TOP,DECLARE,INT,DATETIME},
  sensitive=false,
  morecomment=[l]{--},
  morestring=[b]'
}

\newcommand{\ignore}[1]{}

\definecolor{mygreenua}{HTML}{F1F5EB}
\definecolor{myredda}{HTML}{FFE6E6}

\newcommand{\uahelper}[1]{\colorbox{myredda}{\smaller$\uparrow$#1}}
\newcommand{\dahelper}[1]{\colorbox{mygreenua}{\smaller$\downarrow$#1}}

\newcommand{\uaghelper}[1]{\colorbox{mygreenua}{\smaller$\uparrow$#1}}
\newcommand{\dabhelper}[1]{\colorbox{myredda}{\smaller$\downarrow$#1}}

\newcommand{\ua}[1]{\ifthenelse{\equal{#1}{0.00} \or \equal{#1}{0.000}}{}{\uahelper{#1}}}
\newcommand{\da}[1]{\ifthenelse{\equal{#1}{0.00} \or \equal{#1}{0.000}}{}{\dahelper{#1}}}
\newcommand{\uag}[1]{\ifthenelse{\equal{#1}{0.00} \or \equal{#1}{0.000}}{}{\uaghelper{#1}}}
\newcommand{\dab}[1]{\ifthenelse{\equal{#1}{0.00} \or \equal{#1}{0.000}}{}{\dabhelper{#1}}}

\newtheorem{lemma}{Lemma}
\newtheorem{proposition}{Proposition}

\newcommand{\logoso}{\raisebox{-0.05in}{\includegraphics[width=0.21in]{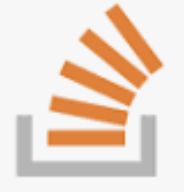}}}
\newcommand{\logoms}{\raisebox{-0.05in}{\includegraphics[width=0.21in]{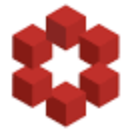}}}
\newcommand{\logoes}{\raisebox{-0.05in}{\includegraphics[width=0.21in]{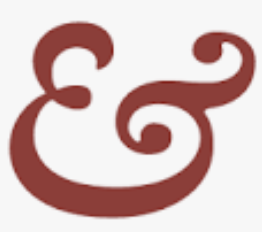}}}
\newcommand{\logosu}{\raisebox{-0.05in}{\includegraphics[width=0.21in]{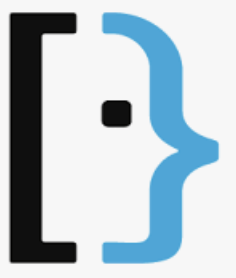}}}

\title{Debiasing Without Protected Attributes: \\ Latent Concept Erasure from Textual Profiles}

\author{Shun Shao${^{\rainbowicon[cbBlue],}}$\thanks{Corresponding author. Code is available at \url{https://github.com/jasonshaoshun/HSAL}.} \qquad  Zheng Zhao${^{\rainbowicon[cbPurple],\dagger}}$ \\  \textbf{Anna Korhonen}$^{\rainbowicon[cbBlue]}$  \qquad
\textbf{Yftah Ziser}$^{\rainbowicon[cbOrange]}$ \qquad \textbf{Shay B. Cohen}${^{\rainbowicon[cbPurple]}}$ \qquad  \\
$^{\rainbowicon[cbBlue]}$University of Cambridge \qquad ${^{\rainbowicon[cbPurple]}}$University of Edinburgh \\
$^{\rainbowicon[cbOrange]}$NVIDIA Research   \\
\texttt{ss3047@cam.ac.uk} \quad \texttt{zheng.zhao@ed.ac.uk} \\
\texttt{alk23@cam.ac.uk} \quad \texttt{yziser@nvidia.com} \quad \texttt{scohen@inf.ed.ac.uk} }

\begin{document}
\maketitle
\begingroup
\renewcommand{\thefootnote}{\textdagger}
\footnotetext{Work conducted while the author was at the University of Edinburgh. The author is now at Meta FAIR.}
\endgroup

\begin{abstract}
Most fairness research in NLP assumes direct access to protected attributes such as gender, race, or nationality. In practice, however, such information is often unavailable due to privacy constraints, missing metadata, or legal restrictions, even though models may infer it from indirect textual cues. This raises a key question: can debiasing succeed without direct access to sensitive attributes? We propose H-\textsc{SAL}, which performs post-hoc concept and attribute erasure using \emph{self-description text} as an implicit debiasing signal. To support this setting, we introduce a multi-domain Stack Exchange-based fairness benchmark for helpfulness prediction that includes both \emph{explicit} and \emph{implicit} signals, enabling comparison between standard debiasing with protected labels and debiasing without access to sensitive information. Across encoder and decoder-only language models,
we find that implicit self-description provides a useful debiasing signal in many evaluated settings, sometimes achieving lower group disparity than explicit-label-based erasure, at a possible cost in helpfulness prediction accuracy.
Our results broaden representation-level fairness research and provide a new benchmark for studying debiasing under realistic data constraints.
\end{abstract}

\section{Introduction}

\begin{figure}[!th]
    \centering
    \begin{mdframed}[backgroundcolor=gray!10, roundcorner=5pt, linewidth=0.5pt]
    \small
    \logoso \textbf{Input ($\mathbf{X}$):} \textit{Q: What is the C\# version of VB.NET's InputBox?} \\
    \textit{A: You mean InputBox? Just look in the Microsoft.VisualBasic namespace. C\# and VB.Net share a common library. If one language can use it\ldots} \\
    \textbf{Answer Helpfulness ($\mathbf{Y}$):} Answer score = 3 \\
    \textbf{About Me ($\mathcal{H}$):} \textit{\colorbox{yellow!20}{Director of Information Technology,} \colorbox{yellow!20}{York University Nebraska.}\ldots The avatar is both because I play counter strike and a nod to lambda expressions in C\#\ldots} \\
    \textbf{Relevant Attributes ($\mathbf{Z}$):} Reputation = 418{,}535; Location = York, Nebraska, USA
    \end{mdframed}

    \vspace{0.08cm}
    
    \begin{mdframed}[backgroundcolor=gray!10, roundcorner=5pt, linewidth=0.5pt]
    \small
    \logoes \textbf{Input ($\mathbf{X}$):} \textit{Q: What is the closest alternative to ``rubbish'' in American English?} \\
    \textit{A: First of all, the word rubbish is pretty well understood by American English speakers, and although it does have a British flavor it is used occasionally by Americans\ldots} \\
    \textbf{Answer Helpfulness ($\mathbf{Y}$):} Answer score = 8 \\[0.08cm]
    \textbf{About Me ($\mathcal{H}$):} \textit{I have a \colorbox{yellow!20}{degree in linguistics} and am partial to descriptivist approaches to usage\ldots I am a \colorbox{yellow!20}{native speaker of American English}\ldots} \\
    \textbf{Relevant Attributes ($\mathbf{Z})$:} Reputation = 69{,}358; Location = San Jose, CA, United States
    \end{mdframed}

    \caption{\textbf{Self-description cues used for textual erasure.} Real examples from Stack Exchange user profiles, where $\mathcal{H}$ denotes the user's self-description. The \colorbox{yellow!20}{highlighted spans} illustrate textual cues such as job titles, institutional affiliations, qualifications, and linguistic background, which may be associated with user attributes such as reputation and location. The entire self-description is used to learn directions for textual erasure.}

    \label{fig:motivating_example}
\end{figure}

Fairness research in natural language processing (NLP) has largely assumed that protected attributes are explicitly available. Under this setup, many methods use demographic labels such as gender, race, or nationality to debias models and measure fairness through group performance gaps. In practice, however, these attributes are often missing or unavailable due to privacy, legal, or data-collection constraints \citep{ashurst2023withoutdemographics}. Yet their absence from the database does not imply their absence from model representations: \citet{neplenbroek2025personalization} show that LLMs can infer demographics from stereotypical dialogue cues, encode them in latent representations, and let them shape responses even without explicit disclosure.

A prominent line of work addresses bias through post-hoc representation erasure under explicit supervision. Methods such as INLP \citep{ravfogel-etal-2020-null}, SAL \citep{shao-etal-2022-SAL}, RLACE \citep{ravfogel22-LACE,shauli-2022-adversarial}, and LEACE \citep{nora-etal-leace} remove demographic information while largely preserving task utility. However, they rely on direct access to demographic labels, which are often unavailable or unreliable in deployment. To relax this requirement, \citet{shao2023erasure} propose zero-shot debiasing by alternating between inferring attribute assignments and removing directions correlated with bias, while \citet{shao-etal-2025-iterative} show that cross-lingual transfer can identify shared bias subspaces for low-resource settings.

This recent work raises a broader question: can concept and attribute erasure succeed using only implicit textual signals rather than explicit protected-attribute labels? More fundamentally, are explicit demographic labels really the strongest supervision for debiasing, or can indirect cues be equally effective? Although prior work has explored proxies such as name embeddings for race and gender \citep{romanov-etal-2019-whats}, direct protected-attribute supervision remains the dominant paradigm.

We study this question through concept and attribute erasure from user self-description text in Stack Exchange (SE) communities (Figure~\ref{fig:motivating_example}). Our implicit signal comes only from free-form self-description, not from structured metadata such as reputation or engagement. We propose H-\textsc{SAL}, a variant of SAL \citep{shao-etal-2022-SAL} that uses profile-text representations in place of explicit attribute labels. Because the same self-description signal captures multiple attribute-related cues, a single erasure step can mitigate several forms of user-attribute bias rather than requiring a separate run per attribute. We compare H-\textsc{SAL} against standard debiasing with explicit protected attributes, allowing us to test both whether implicit supervision is viable and whether it can surpass direct demographic supervision.

Across a new fairness benchmark for helpfulness prediction, we find that self-description text provides a surprisingly strong debiasing signal. In both encoder and decoder-only models, implicit post-hoc erasure often matches or outperforms explicit-label-based debiasing. More broadly, our results show that effective representation-level debiasing need not depend on explicitly observed protected attributes, but can instead emerge from indirect textual cues through which social information is encoded in language.

\section{Related Work}

Much fairness research in NLP assumes access to explicit protected attributes and removes them from learned representations through post-hoc erasure. INLP iteratively projects representations into the null space of attribute classifiers \citep{ravfogel-etal-2020-null}, while RLACE, kernelized concept erasure, SAL, and LEACE extend this paradigm through adversarial optimization, nonlinear removal, spectral methods, and closed-form linear erasure \citep{ravfogel22-LACE,shauli-2022-adversarial,shao-etal-2022-SAL,nora-etal-leace}. Earlier adversarial approaches also suppressed demographic information during training \citep{edwards2015censoring,coavoux-etal-2018-privacy,elazar-goldberg-2018-adversarial,barrett-etal-2019-adversarial}. Together, these methods form a strong debiasing toolkit, but they mostly rely on explicit protected-attribute supervision.

A smaller line of work considers weaker or indirect supervision. \citet{romanov-etal-2019-whats} use name embeddings as proxies for race and gender; AMSAL \cite{shao2023erasure} removes unaligned attributes without paired instance-level labels; \citet{han2023everybody} develop an unsupervised data locality-based method to mitigate bias, and IMSAE \citep{shao-etal-2025-iterative} transfers erasure subspaces across languages in zero-shot settings. Related work includes debiasing without demographics \citep{orgad-belinkov-2023-blind}, proxy-based fairness formulations \citep{gupta2018proxyfairness}, and evidence that LLMs infer demographic traits from subtle cues and encode them in latent representations even without explicit disclosure \citep{neplenbroek2025personalization}. Similar post-hoc editing ideas also appear in text-to-image debiasing, though usually with explicit demographic annotations \citep{fu2025fairimagen}. Recent representation-engineering work has also explored training-free activation editing, including SAE-based methods that identify and modify instruction-relevant sparse features to improve model behavior during generation \citep{zhao2025sparseactivationeditingreliable}.
Our work is related to this broader family of post-hoc representation interventions, but is closest to the second line above: we use only free-form self-description, rather than names or structured metadata, as the implicit signal and compare it against explicit-attribute debiasing.

\section{Problem Formulation and Notation}

We use bold lowercase letters for vectors, and bold uppercase letters for matrices and random variables/vectors. For random variables, $\rv{S}$ and $\rv{T}$, we denote by $\rho(\rv{S}, \rv{T})$ the correlation between them. When they are standardized (mean $0$, variance $1$), their correlation is also their covariance.

Each example consists of a sample from four variables:
\[
\rv{X} \in \mathbb{R}^{d}, \qquad
\rv{Y} \in \mathbb{R}, \qquad
\rv{Z} \in \mathcal{Z}, \qquad
\mathcal{H},
\]
where $\rv{X}$ is the representation of the main-task input, $\rv{Y}$ is the task label, $\rv{Z}$ is the guarded attribute ($\mathcal{Z}$ denotes the set of possible attribute values), and $\mathcal{H}$ is the raw self-description text. We encode $\mathcal{H}$ with a text encoder $\phi$ into a continuous representation
$\rv{H} = \phi(\mathcal{H}) \in \mathbb{R}^{d'}$.

Our goal is to predict $\rv{Y}$ from $\rv{X}$ while reducing the information about $\rv{Z}$ contained in $\rv{X}$. In the explicit setting, debiasing methods have direct access to $\rv{Z}$. In our implicit setting, $\rv{Z}$ is not used by the debiasing method; instead, we use the self-description text $\mathcal{H}$, through its representation $\rv{H}$, as an indirect signal for erasure.

We assume $n$ samples $\left\{ \bigl(\mathbf{x}^{(i)}, y^{(i)}, \mathcal{H}^{(i)} \bigr) \right\}_{i=1}^{n}$,
with corresponding guarded attributes
$\left\{ z^{(i)} \right\}_{i=1}^{n}$.
Let  $\mathbf{h}^{(i)} = \phi\!\left(\mathcal{H}^{(i)}\right) \in \mathbb{R}^{d'}$
denote the representation of the $i$-th self-description. Our method estimates an erasure subspace from paired samples $(\mathbf{x}^{(i)}, \mathbf{h}^{(i)})$.

\section{Methodology}
\label{sec:implicit-sal}
\label{section:method}

\begin{figure}
\centering
\scalebox{0.8}{
    \begin{tikzpicture}

\filldraw[fill=red,   fill opacity=0.4, draw=black] (0,1.1) circle (2);
\filldraw[fill=green, fill opacity=0.4, draw=black] (-1.3,-0.6) circle (2);
\filldraw[fill=blue,  fill opacity=0.4, draw=black] (1.3,-0.6) circle (2);

\node at (0,2.4) {$\rv{X}$};
\node at (-2.4,-0.6) {$\rv{Z}$};
\node at (2.4,-0.6) {$\rv{H}$};

\end{tikzpicture}
}
\caption{\label{fig:venn}A Venn diagram showing the information content overlap between the different r.v.s. The success of H-SAL depends on how much the part of $\rv{Z}$ that overlaps with $\rv{X}$ also overlaps with $\rv{H}$.}
\end{figure}

Building on Spectral Attribute removaL (SAL; \citealt{shao-etal-2022-SAL}), we refer to our approach
as H-SAL (hidden SAL). SAL performs the erasure by projecting the input
representations from $\rv{X}$ into a subspace that is least covarying
with $\rv{Z}$ (represented as a one-hot vector).
Rather than working directly with $\rv{Z}$, H-SAL performs the erasure
implicitly, by building the cross-covariance matrix that relies on $\rv{H}$:
\begin{equation}
\mymat{\Omega} \;=\; \frac{1}{n}\sum_{i=1}^{n} \myvec{x}^{(i)} \bigl(\myvec{h}^{(i)}\bigr)^{\top}
\;\in\; \mathbb{R}^{d \times d'}.
\end{equation}
Once $\mymat{\Omega}$ is calculated, we perform singular value decomposition (SVD)
on it, and then use the projection on samples from $\rv{X}$ based
on $\hat{\rv{X}}$:
\begin{equation}
\hat{\rv{X}} = \left(\mymat{I} - \mymat{U}_{1:k} \mymat{U}_{1:k}^{\top}\right) \rv{X},
\end{equation}
\noindent where  $\mymat{U}_{1:k}$ are the $k$ columns of the left singular vector matrix derived by SVD on $\mymat{\Omega}$ associated with the largest singular values.  Removing these directions reduces the covariance between the main-task and self-description representations, which mitigates attribute-related bias when $\rv{H}$ serves as an informative proxy for $\rv{Z}$. The singular values of $\mymat{\Omega}$ can guide the choice of $k$ without protected-attribute labels: a clear drop in the spectrum suggests a natural cutoff. When the singular values decrease smoothly, $k$ can instead be chosen based on a main-task validation accuracy budget. The fairness and utility trade-off across different values of $k$ is further analyzed in our benchmark.
One advantage of SAL for the implicit approach is that it allows us to
use continuous vector-valued variables as protected attributes, allowing us to replace a categorical $\rv{Z}$ with a vector $\rv{H}$.

This approach builds on prior work using projection methods to debias or erase information from embeddings \cite[\emph{inter alia}]{mu-2018-allbutthetop,liang-etal-2020-monolingual,shao-etal-2022-SAL,ravfogel22-LACE}, with further theoretical justification provided in Appendix~\ref{app:justification}.

\begin{figure}[t]
    \centering
    \includegraphics[width=\linewidth]{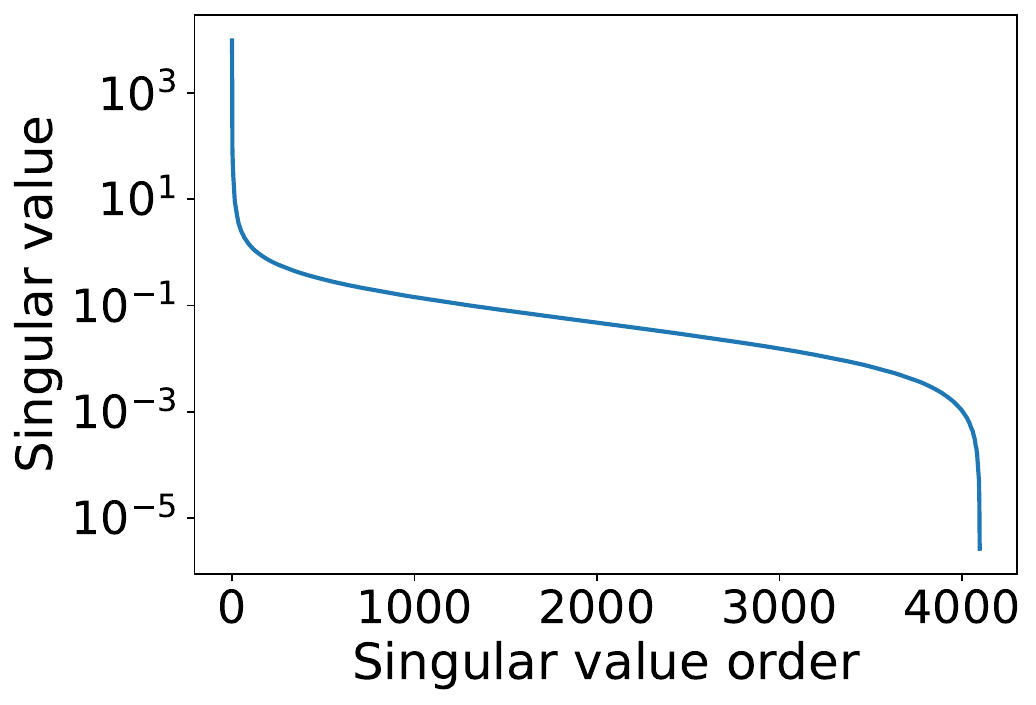}

    \caption{Singular values of the cross-covariance matrix between input representations $\mathbf{X}$ and self-description representations $\mathbf{H}$ for Llama-3.1-8B on the Mathematics dataset. The gradual decay suggests that user-related information is distributed across multiple latent directions, helping explain why H-SAL captures richer bias-related features and achieves stronger debiasing than SAL.}

    \label{fig:singular-value-decay}

\end{figure}

\paragraph{Discussion and Justification}  (The rest of this section can be skimmed in a first reading.) Replacing $\rv{Z}$ with $\rv{H}$ as a mediator has several key advantages:

\begin{itemize}[topsep=0pt,noitemsep,parsep=0pt,leftmargin=*]

\item It does not require direct access to the protected attribute, but rather allows us to access and erase associated information through indirect evidence;

\item It allows greater flexibility in choosing the dimensionality of the erasure subspace, because $k$ can vary and the rank of $\mymat{\Omega}$ is larger than typically with a binary or multiclass attribute $\mymat{Z}$ (if $\mymat{\Omega}$ was computed
directly using $\mymat{Z}$). As shown in Figure~\ref{fig:singular-value-decay}, the singular values decay much more gradually when $\mymat{\Omega}$ is derived from $\rv{H}$, indicating a higher effective rank than when using $\rv{Z}$ directly;

\item It may capture more variability and subtle cues that are not directly captured by a discrete pre-determined $\rv{Z}$, but are captured by $\rv{H}$.
    
\end{itemize}

We note that a central premise of this approach, when the goal
is to erase the $\rv{Z}$ concept, is that if $\rv{X}$ and $\rv{Z}$ covary, and if $\rv{H}$ and $\rv{Z}$ covary, then $\rv{X}$  and $\rv{H}$ covary in a similar, overlapping way. Hence, it is possible to use $\rv{H}$ as an intermediate variable.

Covariance levels are not necessarily retained through intermediate variables. For example, $\rv{H}$ and $\rv{Z}$  could be independent of each other (hence with covariance $0$), but if
$\rv{X} = \mymat{A} \rv{Z} + \mymat{B} \rv{H}$ for  two matrices $\mymat{A} \in \mathbb{R}^{d \times m}, \mymat{B} \in \mathbb{R}^{d \times d'}$ then $\rv{X}$ could covary with both. Or, alternatively, if all variables are Boolean and $\rv{X} = \rv{Z} \textit{\, xor\,  } \rv{H}$, the covariance between $\rv{X}$ and either of the other $\rv{H}$ and $\rv{Z}$ could be $0$ while $\rv{Z}$ and $\rv{H}$ could covary.

However, the intuition that covariance is transferable through an intermediate variable is true to an extent, and for simplicity, assume that the mean of $\rv{X}$, $\rv{H}$ and $\rv{Z}$ is 0 and that the random vectors are normalized
by the square-root inverse of their (full-rank) covariance matrices (in which case, their covariance becomes the identity matrix). Define $S$ to be a function that takes a matrix and returns the largest singular value of the matrix together with the left singular vector and right singular vector associated with it. Then, define:
\begin{align}
    (\rho_1, \myvec{u}_1, \myvec{v}_1) = S(\mathbb{E}[\rv{X}\rv{Z}^{\top}]), \label{eq:a} \\
     (\rho_2, \myvec{u}_2, \myvec{v}_2) = S(\mathbb{E}[\rv{X}\rv{H}^{\top}]), \label{eq:b} \\
    (\rho_3, \myvec{u}_3, \myvec{v}_3) = S(\mathbb{E}[\rv{H}\rv{Z}^{\top}]) . \label{eq:c}
\end{align}
It holds that $\rho_i$ corresponds to the correlation (and covariance, due to the normalization above) between the $\myvec{u}_i$, $\myvec{v}_i$ projection of the corresponding r.v.s in the covariance matrix. The question one might ask is: \emph{can we express (or lower-bound) $\rho_1$ in terms of $\rho_2$ and $\rho_3$ (and all $\rv{u}_i$ and $\rv{v}_i$)?} In cases where $\rv{Z}$ is available in only small amounts, this means it could be labeled through $\rv{H}$ on a small dataset, and then the concept erasure could be applied to $\rv{X}$ through $\rv{H}$ with provided guarantees on the amount removed with respect to $\rv{Z}$. This idea can be formalized as follows. Define the function for $0 \le a,b \le 1$:
\begin{equation}
    f(a,b) = \max \{ 0, ab - \sqrt{(1-a^2)(1-b^2)} \}.
\end{equation}
In Appendix~\ref{app:justification}, we show that under simple conditions on the singular vectors, it holds that:
\begin{equation}
    \rho_1 \ge f(\rho_2, f(\rho_3, \myvec{u}_3^{\top} \myvec{v}_2)).
\end{equation}
Note that the bound can be symmetrically flipped to have $\rho_2$ on the left, dependent on $\rho_1$. Figure~\ref{fig:bound-plot} provides a plot of the bound on $\rho_1$ for two cases: one in which there is a 60 degree angle between $\myvec{u}_3$ and $\myvec{v}_2$ and one in which they are identical.

\begin{figure}
    \includegraphics[width=0.48\textwidth]{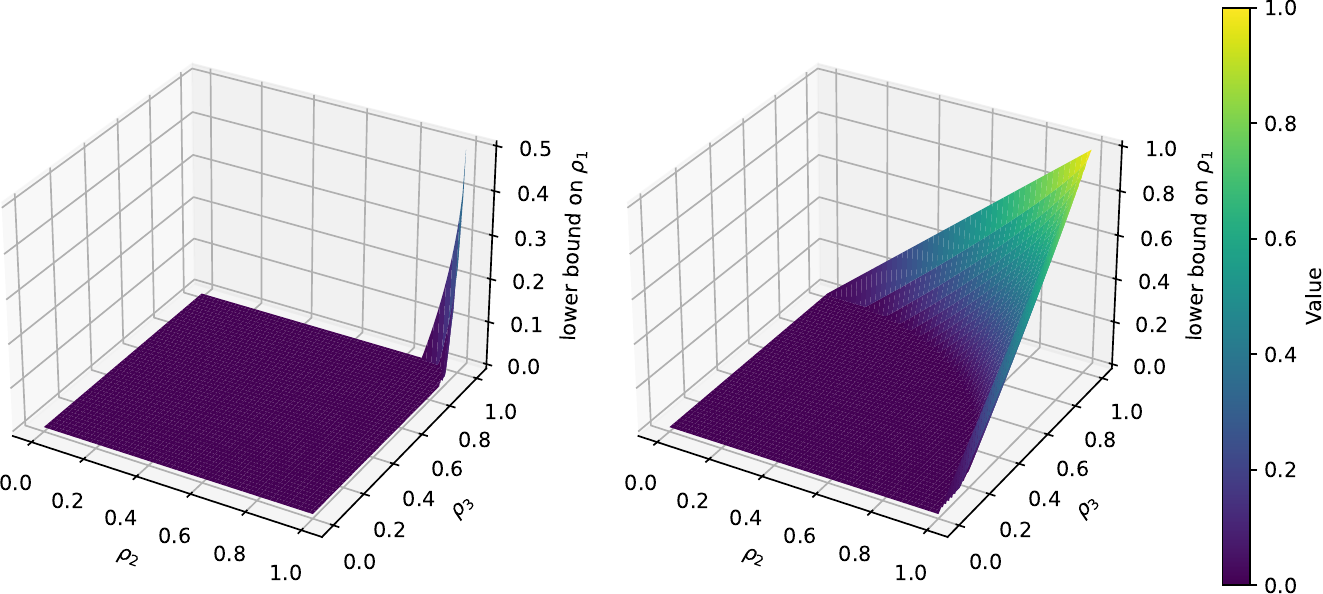}
    \caption{Plots for the bound for $\rho_1$ where $\myvec{u}_3^{\top} \myvec{v}_2 = 0.5$ (left) and $=1$ (right). \label{fig:bound-plot}}
\end{figure}

While we aim to minimize covariance between the projected $\rv{X}$ and $\rv{Z}$ (or $\rv{H}$), we do so by finding the maximally covarying directions, and then project to their complement space. Figure~\ref{fig:venn} provides the intuition in terms of a Venn diagram.

\ignore{
\paragraph{Spectral decomposition and erasure.}
To identify the directions in $X$ that are most entangled with the implicit identity signal in $Z_{\text{text}}$, we compute the singular value decomposition (SVD) of the cross-covariance matrix:
\[
\Omega = U \Sigma V^{\top}.
\]
Here, $U \in \mathbb{R}^{d_x \times r}$ and $V \in \mathbb{R}^{d_z \times r}$ contain the left and right singular vectors, respectively, $r = \mathrm{rank}(\Omega)$, and $\Sigma = \mathrm{diag}(\sigma_1, \sigma_2, \ldots, \sigma_r)$ contains singular values in descending order, with
\[
\sigma_1 \geq \sigma_2 \geq \cdots \geq \sigma_r \geq 0.
\]
The left singular vectors $\{u_j\}_{j=1}^{r}$ define orthogonal directions in the main-task representation space. The leading singular vectors, corresponding to the largest singular values, identify the directions in $X$ that exhibit the strongest linear co-variation with the user's textual profile.

To perform textual erasure, we choose a rank threshold $k$ and form the matrix
\[
U_{1:k} = [u_1, u_2, \ldots, u_k] \in \mathbb{R}^{d_x \times k}.
\]
We then project $X$ onto the orthogonal complement of the corresponding bias subspace:
\[
\hat{X} = \left(I - U_{1:k} U_{1:k}^{\top}\right) X.
\]
This projection removes the top-$k$ components of $X$ that are most strongly aligned, in a linear cross-covariance sense, with the implicit user profile.

\paragraph{Why implicit erasure can be more effective.}

}

\section{The Stack Exchange Fairness Benchmark}
\label{sec:dataset}

Unlike benchmarks that assume explicit protected attributes, real platforms often expose user-authored metadata, such as biographies, that indirectly reveal identity or background.

To study this setting, we construct a new benchmark from SE. Each example consists of a main-task input $\rv{X}$, profile text $\mathcal{H}$, and user attribute labels $\rv{Z}$. Here, $\rv{X}$ is the question--answer context for helpfulness prediction ($\rv{Y}$), and $\mathcal{H}$ is user-provided self-description. The attribute $\rv{Z}$ is used only for evaluation and explicit-label baselines; in the implicit setting, the model must identify and remove the bias signal from profile text alone. Figure~\ref{fig:motivating_example} shows such signals in practice.

The benchmark is designed to test: (1) implicit bias detection; (2) generalization beyond memorized user cues through \emph{user-disjoint} splits; and (3) the separation of model bias from distributional artifacts through \emph{natural} and \emph{balanced} settings.

\subsection{Source Communities and Example Construction}
\label{sec:se_sources}

\paragraph{Domains.}
We construct four subsets from the SE network: \textbf{StackOverflow}, \textbf{Mathematics}, \textbf{SuperUser}, and \textbf{English}. Together, they span diverse content and discourse styles, from technical problem solving ({StackOverflow}, SuperUser) to formula-heavy reasoning ({Mathematics}) and language-focused explanations ({English}), enabling evaluation across different notions of helpfulness.

\paragraph{Instances.}
Each data point is an answer-level example consisting of: (i) the question--answer context as the main-task input $\rv{X}$; (ii) user-authored profile text $\mathcal{H}$ (primarily \texttt{AboutMe}); and (iii) user metadata needed to derive evaluation attributes $\rv{Z}$ (e.g., location, reputation, engagement). We start from public SE data dumps and construct examples by joining post- and user-level tables, then apply a shared text cleaning and filtering pipeline.\footnote{We drop entries with missing or insufficient main text or profile text. Full extraction and preprocessing details are provided in Appendix~\ref{app:data-processing}, with summary statistics reported in Table~\ref{tab:user-counts-natural-disjoint}. Profile availability statistics are provided in Appendix~\ref{app:profile-filtering}.}

\subsection{User Attributes and Implicit Bias Sources}
\label{sec:attributes}

We study five user attributes as sources of bias in helpfulness prediction: three  profile attributes (\textbf{DisplayName}, \textbf{gender}, \textbf{country}) and two platform-derived status signals (\textbf{reputation}, \textbf{engagement}). For evaluation, the main task label \textbf{helpfulness} and the user attributes \textbf{reputation}, \textbf{gender} and \textbf{engagement} are binarized, while \textbf{country} is represented as a five-class attribute across all datasets and \textbf{DisplayName} is represented through clustered name groups.

\paragraph{DisplayName, gender, and country.}
DisplayName, gender, and country act as identity-related proxy signals. DisplayName groups are obtained by clustering user display names in embedding space, with the number of clusters selected by silhouette score. Gender labels are also derived from users' \texttt{DisplayName}: we match extracted names against the Baby Names in England and Wales dataset provided by the Office for National Statistics, using the records from 1996 onward to construct male- and female-name lexicons \citep{ons_babynames_2025}. Country labels are derived from normalized free-form location strings and restricted to the top five most frequent countries in each dataset; the exact country sets are reported in Table~\ref{tab:country-sets}. These attributes may capture cultural, linguistic, gendered, or regional cues that affect perceived credibility beyond answer content.

\paragraph{Reputation and engagement.}
Reputation and engagement are platform-native status signals. Reputation summarizes prior community approval, while engagement reflects participation intensity. Both can induce halo effects, causing models to favor users with signals of status or embeddedness rather than judging the answer itself.

\paragraph{Explicit vs.\ implicit settings.}
The benchmark supports both \emph{explicit} and \emph{implicit} fairness settings. In the explicit setting, the attribute label $\rv{Z}$ is available to the debiasing method. In the implicit setting, only user profile text $\mathcal{H}$ is available, while $\rv{Z}$ is used solely for evaluation. 
Full attribute construction details are deferred to Appendix~\ref{app:attribute-details}.

\subsection{Binarizing Attributes and Labels}
\label{sec:label-construction}

User status attributes and post helpfulness are not binary on SE. We therefore derive binary labels by thresholding platform scores:
\[
\begin{aligned}
\rv{Y} &= \mathbbm{1}[\texttt{AnswerScore} \ge \tau_Y], \\
\rv{Z}_{\text{rep}} &= \mathbbm{1}[\texttt{Reputation} \ge \tau_Z], \\
\rv{Z}_{\text{eng}} &= \mathbbm{1}[\texttt{Engagement} \ge \tau_Z].
\end{aligned}
\]
Here, $\rv{Y}$ is the main-task helpfulness label, while $\rv{Z}_{\text{rep}}$ and $\rv{Z}_{\text{eng}}$ are binary status attributes derived from user reputation and engagement. Country and display name labels are derived via country normalization and display name clustering (\S\ref{sec:attributes}).

\paragraph{Choosing thresholds.}

Thresholds are selected separately for each Stack Exchange community based on the empirical score distributions. For helpfulness, we choose dataset-specific answer-score thresholds so that the downstream label $\rv{Y}$ can be balanced within each split, avoiding majority-class prediction. For reputation and engagement, we choose dataset-specific thresholds that yield non-degenerate status groups while preserving sufficient data for user-disjoint evaluation. The selected thresholds are reported in Table~\ref{tab:selected-thresholds} in the appendix.

For balanced settings, we additionally construct the largest feasible subset in which each joint group $(\rv{Z},\rv{Y}) \in \{0,1\}^2$ occupies 25\% of the data. This allows us to separate model-internal bias from simple label--attribute imbalance. Full threshold grids and selected values are reported in Appendix~\ref{app:split-details}.

\subsection{Data Splits}
\label{sec:splits}

We use a \textbf{User-Disjoint} setting. Splits are constructed by post counts with an approximate 7:1:2 train--validation--test ratio, while enforcing that no \texttt{UserId} appears in more than one split.

The helpfulness label $\rv{Y}$ is balanced in each split to prevent majority-class prediction. For the attribute label $\rv{Z}$, reputation, engagement, and gender are approximately balanced, while display name clusters and country retain their natural distributions because these attributes are highly imbalanced and long-tailed; fully balancing them would discard most examples and yield very small subsets. Further construction details are provided in Appendix~\ref{app:split-details}.

\subsection{Relation to Other SE Datasets}
\label{sec:relation_to_other_SE_dataset}

Prior work has used SE as a source of natural language data with associated user attributes \cite[\emph{inter alia}]{rao-daume-iii-2018-learning,brooke2019condescending,hazoom-etal-2021-text}. The closest to our setting is \citet{shao-etal-2025-iterative}, who introduce the multilingual \textsc{SeFair} dataset. Our benchmark differs in two ways: it studies bias sources available through \emph{textual user metadata}, enabling implicit debiasing from free-form profiles rather than only structured labels, and it includes a true \textbf{user-disjoint} regime, preventing methods from exploiting recurring user-specific cues across data splits.

\section{Experimental Setup}
\label{sec:exp-setup}

We evaluate whether latent concept erasure reduces bias in helpfulness prediction while preserving prediction accuracy. Each experiment instantiates: (i) $\rv{Y}$, corresponding to answer helpfulness; (ii)  $\rv{Z}$, corresponding to one of the five bias sources described in \S\ref{sec:dataset}; and (iii) profile text $\mathcal{H}$ for the implicit signal. %

\paragraph{Representations and encoders.}
We instantiate the framework with three open-source encoders: BERT~\citep{devlin-etal-2019-bert}, Llama-3.1-8B~\citep{grattafiori2024llama3herdmodels}, and Mistral-7B~\citep{jiang2023mistral7b}. For each input text, we extract representations from the final hidden layer. For BERT, we use the first-token (\texttt{[CLS]}) representation; for decoder-only models, we use the final non-padding token representation. These representations are used both for helpfulness prediction and for estimating the erasure subspace.

\paragraph{Debiasing variants.}
We compare four conditions: \textbf{Baseline} (no debiasing), \textbf{SAL (random)} (removing random directions as a control), \textbf{SAL} (estimating the erasure subspace from the explicit attribute label $\rv{Z}$), and \textbf{H-SAL} (estimating the erasure subspace from the implicit profile representation $\rv{H}$ without directly using $\rv{Z}$). After projection, we train a lightweight classifier to predict helpfulness. The random control tests whether any observed change comes from removing an arbitrary subspace rather than one aligned with the target bias. Full classifier and implementation details are provided in Appendix~\ref{app:classifier-details}.

\paragraph{Metrics.}

We report helpfulness accuracy and \textbf{TPR-Gap}, defined as the root mean square (RMS) of class-conditional disparities across attribute groups. For each helpfulness class, we compute twice the standard deviation of the corresponding true positive rates across groups, and then aggregate the two class-specific gaps using RMS. Thus, the metric captures disparities for both Helpful and Unhelpful predictions. Lower values indicate smaller group disparities. Full details are provided in Appendix~\ref{app:evaluation-metrics}.
We report bootstrap uncertainty estimates in Appendix~\ref{app:bootstrap}.

\paragraph{Erasure strength.}
Erasure requires choosing the number of removed directions, $k$. In the main results, H-SAL and random removal use $k=50$, while explicit SAL uses its attribute-specific rank. We examine different values of $k$ in \S\ref{sec:effect_of_k} and provide a matched-$k=50$ comparison in Appendix~\ref{app:matched-k}.

\paragraph{Predictability of $\rv{Z}$ from $\rv{H}$.}
As discussed in \S\ref{section:method}, our method assumes that the representation $\rv{H}$ contains information about the protected attribute $\rv{Z}$; otherwise, there would be little signal to remove. Table~\ref{fig:accuracy-on-self-description} confirms this assumption, shows that profile-text embeddings encode attribute-related information, and provides a justification for studying debiasing in the representation space.

\section{Experiments}

We evaluate three models and compare \textbf{implicit} and \textbf{explicit} textual erasure against a \textbf{random} baseline. Our main results are averaged over the three models, providing a compact view of the overall fairness--utility trade-off across datasets and attributes. Appendix~\ref{app:additional_results} reports per-model results and sensitivity to the number of removed directions.

\begin{table*}[th!]
\centering
\resizebox{\textwidth}{!}{
\begin{tabular}{ll rr  rr  rr  rr}
\toprule
Bias Type & Method & \multicolumn{2}{c}{\textcolor{black}{\logoso \textbf{StackOverflow}}} & \multicolumn{2}{c}{\textcolor{black}{\logoms \textbf{Mathematics}}} & \multicolumn{2}{c}{\textcolor{black}{\logosu \textbf{SuperUser}}} & \multicolumn{2}{c}{\textcolor{black}{\logoes \textbf{English}}} \\
& & Main & TPR-Gap & Main & TPR-Gap & Main & TPR-Gap & Main & TPR-Gap \\
\midrule
\multirow{4}{*}{DisplayName} & Baseline & 69.2 & 4.9 & 64.0 & 4.8 & 71.1 & 4.1 & 63.5 & 9.9 \\
\addlinespace[4pt]
& SAL (random) & \dab{0.4} 68.8 & \da{0.7} 4.2 & \dab{0.2} 63.8 & \da{0.4} 4.4 & 71.1 & \ua{0.1} 4.3 & \dab{1.0} 62.5 & \da{1.5} 8.4 \\
\addlinespace[4pt]
& SAL & \dab{0.8} 68.4 & \da{2.9} 2.0 & \dab{1.1} 62.8 & \da{2.0} 2.8 & \dab{0.3} 70.9 & \ua{0.3} 4.4 & \dab{0.8} 62.7 & \da{0.8} 9.1 \\
\addlinespace[4pt]
& H-SAL & \dab{1.7} 67.5 & \da{3.6} 1.3 & \dab{1.9} 62.1 & \da{3.2} 1.6 & \dab{1.0} 70.1 & \da{1.2} 2.9 & \dab{2.1} 61.5 & \da{3.7} 6.2 \\
\addlinespace[4pt]
\midrule

\multirow{4}{*}{Gender} & Baseline & 56.7 & 0.3 & 54.7 & 2.8 & 58.8 & 11.6 & 54.3 & 9.9 \\
\addlinespace[4pt]
& SAL (random) & \dab{0.8} 55.9 & \ua{0.5} 0.8 & \uag{0.3} 55.0 & \da{0.1} 2.8 & \dab{1.6} 57.2 & \da{0.8} 10.7 & \dab{1.3} 52.9 & \ua{1.8} 11.7 \\
\addlinespace[4pt]
& SAL & \dab{0.5} 56.1 & \ua{0.5} 0.8 & \uag{0.4} 55.1 & \da{0.1} 2.8 & \uag{0.4} 59.1 & 11.6 & \uag{0.1} 54.4 & \da{0.1} 9.8 \\
\addlinespace[4pt]
& H-SAL & \dab{2.4} 54.3 & \ua{1.2} 1.5 & \dab{1.4} 53.3 & \da{1.5} 1.4 & \dab{0.7} 58.1 & \da{1.2} 10.4 & \dab{1.8} 52.4 & \da{0.4} 9.5 \\
\addlinespace[4pt]
\midrule

\multirow{4}{*}{Country} & Baseline & 67.4 & 11.5 & 61.0 & 9.7 & 63.4 & 7.0 & 59.3 & 10.1 \\
\addlinespace[4pt]
& SAL (random) & \dab{0.3} 67.0 & \da{0.7} 10.8 & \dab{0.1} 60.9 & \da{0.5} 9.3 & \dab{0.2} 63.2 & \da{0.1} 6.9 & 59.3 & \da{0.8} 9.2 \\
\addlinespace[4pt]
& SAL & \dab{0.6} 66.8 & \da{8.1} 3.4 & \dab{0.1} 60.9 & \da{2.7} 7.1 & \dab{0.3} 63.1 & \da{3.9} 3.0 & \uag{0.1} 59.5 & \da{0.3} 9.8 \\
\addlinespace[4pt]
& H-SAL & \dab{2.3} 65.1 & \da{9.5} 2.0 & \dab{3.9} 57.1 & \da{5.5} 4.2 & \dab{0.9} 62.5 & \da{3.5} 3.5 & \dab{1.5} 57.8 & \da{0.6} 9.4 \\
\addlinespace[4pt]
\midrule

\multirow{4}{*}{Reputation} & Baseline & 69.2 & 8.6 & 64.0 & 11.8 & 71.1 & 8.8 & 63.5 & 9.8 \\
\addlinespace[4pt]
& SAL (random) & \dab{0.4} 68.8 & \da{0.9} 7.7 & \dab{0.2} 63.8 & \da{0.8} 11.1 & 71.1 & \ua{0.1} 8.9 & \dab{1.0} 62.5 & \da{0.8} 9.0 \\
\addlinespace[4pt]
& SAL & \dab{0.8} 68.4 & \da{8.4} 0.2 & \dab{1.2} 62.7 & \da{10.0} 1.8 & \dab{0.8} 70.3 & \da{4.9} 3.9 & \dab{1.3} 62.2 & \da{4.9} 4.9 \\
\addlinespace[4pt]
& H-SAL & \dab{1.7} 67.5 & \da{8.3} 0.3 & \dab{1.9} 62.1 & \da{10.1} 1.7 & \dab{1.0} 70.1 & \da{5.5} 3.2 & \dab{2.1} 61.5 & \da{4.9} 4.9 \\
\addlinespace[4pt]
\midrule

\multirow{4}{*}{Engagement} & Baseline & 68.6 & 7.7 & 63.6 & 8.5 & 67.2 & 8.4 & 63.6 & 9.6 \\
\addlinespace[4pt]
& SAL (random) & \dab{0.4} 68.2 & \da{1.0} 6.7 & \dab{0.2} 63.4 & \da{0.9} 7.6 & \dab{0.1} 67.1 & \da{0.7} 7.8 & \dab{0.6} 63.0 & \da{2.0} 7.6 \\
\addlinespace[4pt]
& SAL & \dab{1.1} 67.5 & \da{7.2} 0.5 & \dab{1.1} 62.5 & \da{6.4} 2.1 & \dab{0.4} 66.8 & \da{5.4} 3.1 & \dab{1.3} 62.2 & \da{6.0} 3.6 \\
\addlinespace[4pt]
& H-SAL & \dab{2.0} 66.6 & \da{7.2} 0.5 & \dab{2.5} 61.2 & \da{6.7} 1.8 & \dab{0.7} 66.5 & \da{4.9} 3.6 & \dab{1.9} 61.7 & \da{3.1} 6.5 \\
\addlinespace[4pt]
\bottomrule
\end{tabular}
}

\caption{\textbf{Average results under the user-disjoint split.}
Results are averaged across BERT, Llama-3.1-8B, and Mistral-7B over five bias attributes and four Stack Exchange domains. \textbf{Main} denotes helpfulness prediction accuracy, and \textbf{TPR-Gap} denotes generalized disparity, computed as the RMS of class-specific gaps, each defined as \(2 \times\) the standard deviation of true positive rates across groups. \textsc{SAL} uses bias-specific ranks, while \textsc{SAL} (random) and H-\textsc{SAL} use \(k=50\) removed directions for all attributes. The helpfulness label is balanced in each split; reputation, engagement, and gender are approximately balanced, while display-name clusters and country retain their natural distributions. Baseline denotes no erasure.}

\label{tab:ste-summary-average-natural-disjoint-k50}
\end{table*}

\subsection{Main Results}

Table~\ref{tab:ste-summary-average-natural-disjoint-k50} reports model-averaged results in the natural user-disjoint setting across BERT, Llama-3.1-8B, and Mistral-7B. We focus on model averages to highlight the overall pattern across domains and bias types; full per-model results appear in Appendix~\ref{app:eachmodel}. Across all five attributes, the random baseline stays close to the original model, with only small changes in TPR-Gap, indicating that the fairness gains of SAL and H-SAL do not come from removing arbitrary directions. Reductions in TPR-Gap are sometimes accompanied by modest drops in main-task accuracy; because H-SAL removes a larger subspace ($k=50$) than SAL (bias-specific rank), it tends to trade slightly more accuracy for fairness, a trade-off we revisit in \S\ref{sec:effect_of_k}.

Overall, \emph{H-SAL is competitive with, and often outperforms, SAL}, even though SAL uses the protected attribute directly. The advantage is clearest for display name, country and reputation in the technical domains, while gender and engagement show a mixed picture. An explanation is that the explicit attribute is often a coarse summary, whereas profile text encodes a richer set of correlated cues, such as place names, affiliations, seniority markers, and stylistic signals. Consistent with this interpretation, the ablation in \S\ref{sec:effect_of_k} shows that H-SAL often continues to improve as more directions are removed, while random removal stays flat.

\paragraph{Display Name.}

Display name is the strongest case for implicit erasure: H-SAL achieves the lowest TPR-Gap in all four domains. On StackOverflow it reduces the gap from 4.9\% to 1.3\%, compared with 2.0\% for SAL, and Mathematics shows the same ordering (4.8\% to 1.6\%, versus 2.8\%). The contrast is sharpest on SuperUser, where the baseline gap is small (4.1\%) and neither random nor SAL reduces it (4.3\% and 4.4\%, respectively), while H-SAL still lowers it to 2.9\%. On English, H-SAL reduces the gap from 9.9\% to 6.2\%, again ahead of SAL at 9.1\%. Profile text therefore provides a consistently useful signal for name-related bias, even where the explicit label does not.

\paragraph{Gender.}
Gender shows the smallest baseline disparities of all five attributes, and on StackOverflow the gap is already negligible (0.3\%), so there is little bias to remove and all methods keep it small. On the remaining domains, H-SAL gives the lowest TPR-Gap: it reduces the gap from 2.8\% to 1.4\% on Mathematics, and produces smaller reductions on SuperUser (11.6\% to 10.4\%) and English (9.9\% to 9.5\%), whereas SAL leaves the gap essentially unchanged. The gender results are thus modest and reflect a weak baseline signal rather than a failure of erasure; main-task accuracy is also lower in the gender setting, indicating an intrinsically harder, attribute-balanced task.

\paragraph{Country.}

For country, H-SAL produces the largest reductions in the two technical domains: it lowers the gap from 11.5\% to 2.0\% on StackOverflow and from 9.7\% to 4.2\% on Mathematics, in both cases well below SAL (3.4\% and 7.1\%). On SuperUser the two methods are close, with SAL marginally lower (3.0\% versus 3.5\%). English is a harder case: the baseline gap (10.1\%) barely moves under any method, including SAL. This pattern suggests that location-related cues are present in user biographies in a sufficiently systematic form on StackOverflow and Mathematics, but that the signal is much weaker, or more entangled with task content, in the English community.

\paragraph{Reputation.}

Reputation is a case where H-SAL and SAL are essentially on par, and both far outperform random removal. H-SAL gives the lowest or tied-lowest TPR-Gap in three of the four domains: it matches SAL on English (both 4.9\%), slightly improves on it on Mathematics (1.7\% versus 1.8\%), and is clearly better on SuperUser (3.2\% versus 3.9\%). StackOverflow is the only domain where SAL is marginally ahead (0.2\% versus 0.3\%). In all cases the gap is reduced substantially from a baseline of roughly 9--12\%, showing that profile text recovers status-related bias about as effectively as the explicit reputation label.

\begin{figure*}[t]
  \centering
  \begin{subfigure}[t]{0.33\textwidth}
    \centering
    \includegraphics[width=\linewidth]{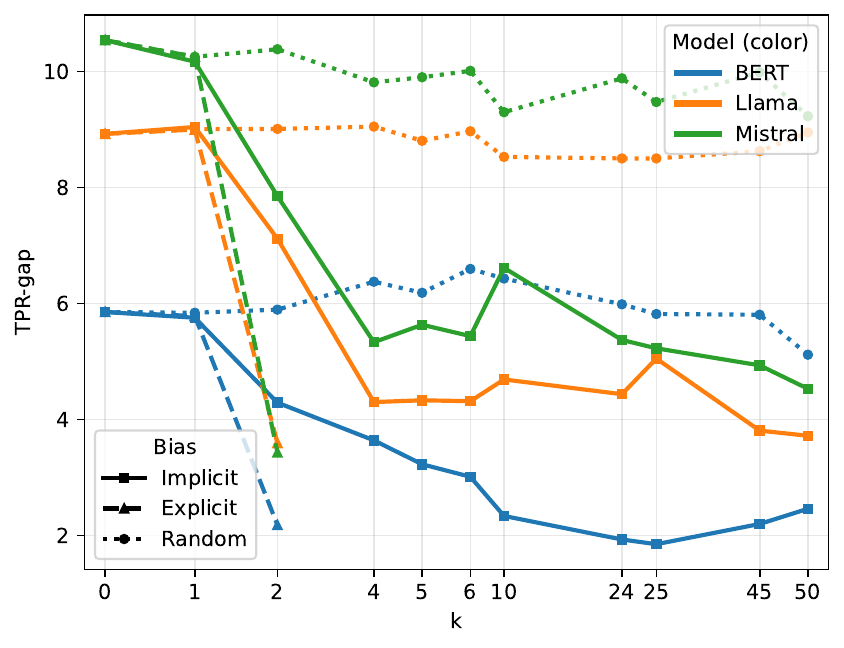}
    \caption{SuperUser}
  \end{subfigure}\hfill
  \begin{subfigure}[t]{0.33\textwidth}
    \centering
    \includegraphics[width=\linewidth]{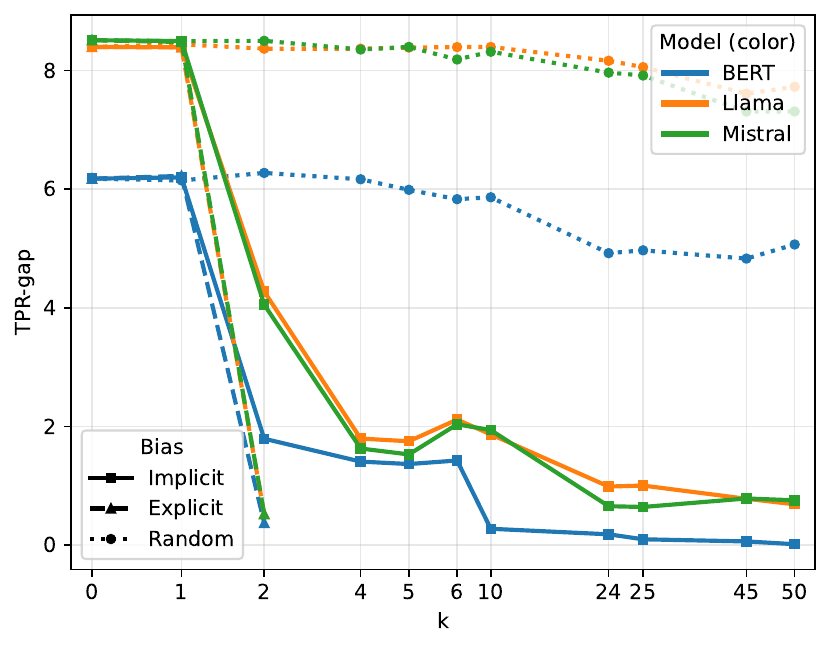}
    \caption{StackOverflow}
  \end{subfigure}\hfill
  \begin{subfigure}[t]{0.33\textwidth}
    \centering
    \includegraphics[width=\linewidth]{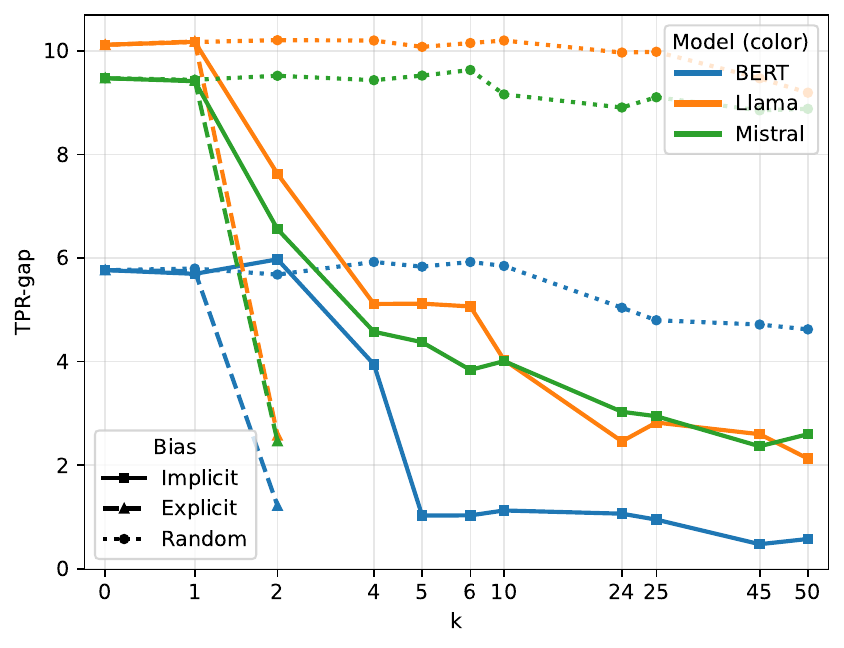}
    \caption{Mathematics}
  \end{subfigure}
\caption{\textbf{Engagement bias reduction across three Stack Exchange communities under the natural-disjoint split.}
Each subfigure plots TPR-Gap against the number of removed directions ($k$) for \texttt{BERT}, \texttt{Mistral-7B}, and \texttt{Llama-3.1-8B}, comparing random, explicit, and implicit \textsc{SAL}. \textbf{TPR-Gap} is the RMS of class-specific disparities ($2 \times$ StdDev across groups); for binary attributes, each class-specific gap equals the absolute difference in group true positive rates. Lower values indicate smaller disparity.}
  \label{fig:tpr-gap_agaist_k}
\end{figure*}

\paragraph{Engagement.}

Engagement gives a more mixed outcome. H-SAL is best or tied on StackOverflow (both methods reach 0.5\%, from a baseline of 7.7\%) and Mathematics (1.8\% versus 2.1\%), but SAL is stronger on SuperUser (3.1\% versus 3.6\%) and also on English, where SAL reaches 3.6\% while H-SAL only lowers the gap to 6.5\%. Both methods improve substantially over random removal in every domain. Engagement therefore illustrates that the implicit signal is not uniformly superior: when participation-related cues are weakly expressed in self-description, the explicit label can still be the more reliable supervision source.

\subsection{Effect of the Number of Removed Directions}
\label{sec:effect_of_k}

Table~\ref{tab:ste-summary-average-natural-disjoint-k50} reports results under a strong erasure setting, \(k=50\). Figure~\ref{fig:tpr-gap_agaist_k} shows TPR-Gap as \(k\) varies for engagement bias across communities and backbones; additional model-averaged results for \(k \in \{5,10,24,45\}\) are provided in Appendix~\ref{app:different-k}. As \(k\) increases, H-\textsc{SAL} generally reduces TPR-Gap more than the random baseline, indicating that improvements come from removing bias-aligned rather than arbitrary directions. 
This also helps explain why H-\textsc{SAL} can outperform explicit \textsc{SAL}: profile text spreads bias across many correlated directions (Figure~\ref{fig:singular-value-decay}), while the explicit label spans only a narrow subspace.
Since larger \(k\) may also affect helpfulness accuracy, \(k\) controls the fairness--utility trade-off. 
We further compare explicit SAL and H-SAL at the same $k=50$ in Appendix~\ref{app:matched-k}.

\section{Conclusion}

We study fairness in a setting where protected attributes are not directly available, but may still be implicitly encoded in user-authored text. We introduce a multi-domain Stack Exchange-based benchmark that includes both explicit attribute labels and implicit self-description, alongside a method that performs erasure once from implicit self-description rather than requiring a separate run per predefined attribute. We find that implicit self-description often provides a strong debiasing signal, matching or outperforming explicit-label-based erasure across most settings. These results broaden representation-level fairness research and suggest that indirect textual cues can serve as an effective basis for fairness interventions in realistic settings where demographic labels are absent.

\section*{Limitations}

Our study has several limitations. First, although the empirical results are strong, we evaluate only one family of post-hoc debiasing methods, namely SAL and its implicit variant. Other representation-level debiasing methods, such as INLP, RLACE, LEACE, or newer post-hoc editing approaches, may behave differently in the implicit setting. We note, however, that SAL lends itself very naturally to our setup, which is the reason we use it.

Second, our benchmark focuses on helpfulness prediction in Stack Exchange communities. While this provides a realistic and diverse testbed, the findings may not transfer directly to other tasks, domains, or interaction settings. In particular, the relationship between user self-description and socially salient attributes may vary substantially across platforms. We note, however, that our dataset spans four distinct communities, each with substantially different discourse styles.

Third, our implicit signal is restricted to self-description text. This is a deliberate design choice that isolates the role of free-form textual metadata, but it leaves open how other realistic implicit signals, such as writing style, interaction history, or broader profile context, might affect debiasing.

\section*{Ethical Considerations}

Our work is motivated by the fact that protected attributes are often unavailable in practice due to privacy, legal, and data-collection constraints. The goal of our method is not to infer or reconstruct sensitive user information, but to reduce the extent to which models rely on such information when making predictions. At the same time, studying implicit bias necessarily involves attributes that may be socially sensitive or correlated with identity. This creates two ethical risks. First, user-authored profile text may contain personal information, and even aggregate analysis of such text should be handled with care. Second, methods for detecting or removing implicit bias signals could be misused to profile users or infer private attributes. We therefore position our benchmark and methods strictly as tools for fairness evaluation and mitigation, not for demographic inference or user profiling.

A further consideration is that the attributes studied in this paper, especially \textit{display name} and \textit{country}, are proxies rather than ground-truth demographic labels. They should not be interpreted as definitive indicators of identity. Likewise, reducing disparity on a benchmark does not guarantee fairness in all downstream settings, since fairness is multi-dimensional and depends on the deployment context. We therefore view this benchmark as a controlled testbed for studying one important aspect of fairness under missing protected attributes, rather than as a complete solution to fairness in real-world systems.

\section*{Acknowledgements}
We thank the Stack Exchange community members for enriching the internet in a positive, constructive manner. If it were not for them, we would not appreciate as much how peculiar, diverse and complex human thought can be.
The authors acknowledge the use of resources provided by the Isambard-AI National AI Research Resource (AIRR). Isambard-AI is operated by the University of Bristol and is funded by the UK Government’s Department for Science, Innovation and Technology (DSIT) via UK Research and Innovation; and the Science and Technology Facilities Council [ST/AIRR/I-A-I/1023]~\citep{mcintoshsmith2024isambardaileadershipclasssupercomputer}. SS and AK acknowledge the support of the UKRI Frontier grant EP/Y031350/1.

\bibliography{custom}

\appendix

\newpage

\section{More Details on Justification}
\label{app:justification}

\begin{lemma}
Let $\rv{R}$, $\rv{T}$ and $\rv{S}$ be three scalar random variables. Define:
\begin{equation}
    f(a,b) = \max \{ 0, ab - \sqrt{(1-a^2)(1-b^2)} \}.
\end{equation}
If the correlation between all pairs of the three random variables is non-negative, then:
\begin{equation}
    \rho(\rv{R}, \rv{S}) \ge f(\rho(\rv{R}, \rv{T}), \rho(\rv{T}, \rv{S})).
\end{equation}
    \label{lemma:rho}
\end{lemma}
\begin{proof}
    Assuming a joint probability space on $\rv{R}$, $\rv{T}$ and $\rv{S}$, the Pearson correlation
    between any pair of them could be thought of as cosine of angles between the pair using the
    $L^2$ inner product induced by the probability measure over the random variables. The
    underlying angles are always in $[0,\pi/2]$ due to the correlation being positive.
    For any two angles $\alpha,\beta \in [0,\pi/2]$, it holds that $\cos(\alpha + \beta) = \cos(\alpha)\cos(\beta) - \sin(\alpha)\sin(\beta)$.
    In addition, it holds that $\sin(\alpha) = \sqrt{1-\cos^2(\alpha)}$.
    Replacing $\alpha=\arccos(\rho(\rv{R}, \rv{T}))$ and $\beta = \arccos(\rho(\rv{T}, \rv{S}))$ in the above
    and noticing that the angle between $\rv{R}$ and $\rv{S}$ is smaller than the
    sum of angles between $\rv{R}$ and $\rv{T}$ and $\rv{T}$ and $\rv{S}$, the inequality holds (due to the decreasing monotonicity of $\cos$ on $[0,\pi]$). (All angles are in $[0,\pi/2]$ due to the positivity of the correlation, therefore $\cos(\gamma) \ge \cos(\alpha+\beta)$ for $\gamma =  \arccos(\rho(\rv{R}, \rv{S}))$.)
\end{proof}
This is a repeat of the proof and the discussion of this lemma on Mathematics Stack Exchange at \url{https://tinyurl.com/pzn4uf58}.
\noeqref{eq:b}

We note that $f$ is monotonically non-decreasing in both arguments.
Using the definitions in Eq.~\refeq{eq:a}--\refeq{eq:c} and Lemma~\ref{lemma:rho}, together with a composition of $f$ with itself, we obtain:

\begin{proposition}
Assume $\myvec{u}_3^{\top} \myvec{v}_2 \ge 0$ and that $\rho(\myvec{v}_2^{\top} \rv{H}, \myvec{v}_3^{\top}\rv{Z}) \ge 0$.
It holds that
\begin{equation}
    \rho_1 \ge f(\rho_2, f(\rho_3, \myvec{u}_3^{\top} \myvec{v}_2)).
\end{equation}
\label{prop:1}
\end{proposition}
\begin{proof}
We know the correlation of projected random vectors
is maximized when using the first left/right singular vector, and therefore:
    \begin{equation}
    \rho_1 \ge \underbrace{\rho(\myvec{u}_2^{\top}\rv{X}, \myvec{v}_3^{\top}\rv{Z})}_{\hat{\rho}}.
    \end{equation}
    By Lemma~\ref{lemma:rho} it holds that:
    \begin{equation}
    \hat{\rho} \ge f(\rho_2, \rho(\myvec{v}_2^{\top} \rv{H}, \myvec{v}_3^{\top}\rv{Z})).
\end{equation}
    By the monotonicity of $f$ in both arguments over $[0,1]$ and the definition of $\rho_3$ we can use Lemma~\ref{lemma:rho} again (with $\rv{T} = \myvec{u}_3^{\top} \rv{H}$) and replace the second argument
    to get:
    \begin{equation}
        \hat{\rho} \ge f(\rho_2, f(\rho_3,  \rho(\myvec{u}_3^{\top} \rv{H}, \myvec{v}_2^{\top} \rv{H}))),
    \end{equation}
    Since we assume $\rv{H}$ is normalized, it holds that the last $\rho$ argument equals $\myvec{u}_3^{\top} \myvec{v}_2$.
    Putting it all together we obtain that
    \begin{equation}
        \rho_1 \ge f(\rho_2, f(\rho_3, \myvec{u}_3^{\top} \myvec{v}_2))
    \end{equation}
    as required.
\end{proof}

We note that due to the degree of freedom in the sign of the left and right singular vectors in SVD, we can relax the condition that $\myvec{u}_3^{\top} \myvec{v}_2$ and $\rho(\myvec{v}_2^{\top} \rv{H}, \myvec{v}_3^{\top}\rv{Z})$ are both non-negative to them having the same sign.

\section{Dataset Construction Details}
\label{app:data-processing}

\begin{table*}[th!]
\centering
\resizebox{0.99\textwidth}{!}{
\begin{tabular}{llccccc}
\toprule
Dataset & Split & Reputation & Gender & Location & Engagement & DisplayName \\
\midrule
\multirow{3}{*}{\logoso StackOverflow} & Train & 69976 / 69976 & 2894 / 2894 & 45088 / 16428 / 16767 / 35975 / 13877 & 104839 / 104839 & 17087 / 52472 / 27583 / 8349 / 34461 \\
& Validation & 10017 / 10017 & 716 / 716 & 7456 / 1876 / 2168 / 4969 / 1786 & 15059 / 15059 & 2501 / 7799 / 3861 / 1110 / 4763 \\
& Test & 19967 / 19967 & 854 / 854 & 12027 / 4934 / 4001 / 9896 / 3994 & 29933 / 29933 & 4986 / 15058 / 7900 / 2298 / 9692 \\

\addlinespace[3pt]
\midrule

\multirow{3}{*}{\logoms Mathematics} & Train & 33273 / 33273 & 3616 / 3616 & 34836 / 7755 / 8060 / 12447 / 14074 & 32913 / 32913 & 5949 / 4021 / 15362 / 26265 / 14949 \\
& Validation & 4770 / 4770 & 1112 / 1112 & 5345 / 942 / 1563 / 1544 / 1727 & 4619 / 4619 & 797 / 715 / 2250 / 3694 / 2084 \\
& Test & 9452 / 9452 & 1140 / 1140 & 11788 / 2692 / 2082 / 3762 / 3963 & 9361 / 9361 & 1457 / 1193 / 4204 / 7784 / 4266 \\

\addlinespace[3pt]
\midrule

\multirow{3}{*}{\logosu SuperUser} & Train & 9601 / 9601 & 456 / 456 & 9617 / 5257 / 2354 / 2587 / 2044 & 9251 / 9251 & 2213 / 2052 / 4652 / 3486 / 6799 \\
& Validation & 1379 / 1379 & 194 / 194 & 1121 / 1100 / 377 / 305 / 313 & 1335 / 1335 & 279 / 282 / 644 / 501 / 1052 \\
& Test & 2746 / 2746 & 188 / 188 & 1998 / 1856 / 1258 / 851 / 566 & 2640 / 2640 & 581 / 514 / 1409 / 1009 / 1979 \\

\addlinespace[3pt]
\midrule

\multirow{3}{*}{\logoes English} & Train & 2882 / 2882 & 616 / 616 & 4425 / 4770 / 1276 / 736 / 308 & 2838 / 2838 & 1798 / 579 / 682 / 1474 / 1231 \\
& Validation & 415 / 415 & 278 / 278 & 490 / 766 / 90 / 59 / 6 & 406 / 406 & 257 / 80 / 97 / 237 / 159 \\
& Test & 812 / 812 & 266 / 266 & 1671 / 1107 / 329 / 118 / 161 & 806 / 806 & 407 / 230 / 230 / 459 / 298 \\

\bottomrule
\end{tabular}
}

\caption{\textbf{Number of answer-level examples in each split by protected-attribute class under the natural user-disjoint split.} For binary attributes, entries are reported as class counts in the form $Z=0$ / $Z=1$. For location, entries are reported as counts for each country class. For DisplayName, entries are reported as counts for each of the display name clusters.}
\label{tab:user-counts-natural-disjoint}
\end{table*}

\subsection{Data Extraction and Preprocessing}
\label{app:preprocessing-details}

We start from the public Stack Exchange data dumps for each community and construct answer-level examples by joining post- and user-level tables. For each answer, we retrieve: (i) the answer text and its score; (ii) the associated question context, such as title and tags when available; and (iii) user metadata including \texttt{UserId}, \texttt{DisplayName}, \texttt{AboutMe}, \texttt{Location}, \texttt{Reputation}, and interaction statistics used to compute engagement. These fields are aligned into a post-level table where each row corresponds to one answer and contains both the main-task text and the user profile text $\mathcal{H}$.

The main-task text is constructed by concatenating the cleaned question title and cleaned answer body into a single prompt of the form \texttt{Question: <title>, Answer: <answer>}. This text is then encoded to obtain the representation $\rv{X}$ used for helpfulness prediction and debiasing.

All textual fields, including question and answer content, self-descriptions, and location strings, are processed with a shared cleaning pipeline. We remove HTML markup and normalize whitespace and formatting to obtain cleaned text fields. We discard entries where either the main answer text or the user profile text is empty after cleaning, and filter out very short texts that provide insufficient signal for representation learning or probing.\footnote{We apply the same validity checks across domains to keep preprocessing consistent.}

Table~\ref{tab:user-counts-natural-disjoint} provides statistics on the dataset.

\subsection{Self-description Availability}
\label{app:profile-filtering}

\begin{table}[t]
\centering
\small
\resizebox{\columnwidth}{!}{
\begin{tabular}{lrrrr}
\toprule
Domain & Candidates & Filtered & Retained & Retention \\
\midrule
StackOverflow & 388,618 & 41,338 & 347,280 & 89.36\% \\
Mathematics & 25,255 & 2,691 & 22,564 & 89.34\% \\
SuperUser & 28,449 & 2,971 & 25,478 & 89.56\% \\
English & 6,335 & 507 & 5,828 & 92.00\% \\
\bottomrule
\end{tabular}}
\caption{Self-description availability among candidate users before attribute-specific benchmark construction. Filtered users have missing or insufficient profile text.}
\label{tab:profile-filtering}
\end{table}

Overall, 89--92\% of candidate users have usable self-descriptions. The final numbers of users and posts differ across bias types because each benchmark applies additional attribute-specific selection rules, such as valid normalized country information or unambiguous name matching.

\subsection{Attribute Construction}
\label{app:attribute-details}

We study five user attributes as sources of bias in helpfulness prediction: \textbf{display name}, \textbf{gender}, \textbf{country}, \textbf{reputation}, and \textbf{engagement}. Three of these (\textbf{display name}, \textbf{gender}, and \textbf{country}) are socially salient profile attributes, while the other two (\textbf{reputation}, \textbf{engagement}) are platform-derived status signals. We do not treat them as definitive ground-truth demographics; instead, they define evaluation groups that may correlate with writing style, topical focus, visibility, and community norms, and may therefore be spuriously predictive of helpfulness even when answer quality is comparable.

\paragraph{Display Name.}

Display names can encode cultural, linguistic, or stylistic cues that models may exploit. We simulate this source of bias by clustering user names in embedding space and treating the resulting clusters as evaluation groups. We select the number of clusters using the average silhouette score and use five clusters in our experiments.

\paragraph{Country.}
Country labels are derived from users' free-form location strings. Because these strings are noisy and heterogeneous, we normalize them to standardized country labels using a location-resolution pipeline and retain only cases that resolve unambiguously to a single country. To ensure sufficient support per group, we restrict each domain to its most frequent countries after normalization. We provide the resulting country labels in Table~\ref{tab:country-sets}.

\begin{table}[th!]
\centering
\resizebox{\columnwidth}{!}{
\begin{tabular}{ll}
\toprule
Dataset & Retained country labels \\
\midrule
Mathematics & USA, FRA, CAN, DEU, GBR \\
StackOverflow & USA, IND, DEU, GBR, CAN \\
English & GBR, USA, CAN, IND, ITA \\
SuperUser & USA, GBR, CAN, DEU, IND \\
\bottomrule
\end{tabular}
}
\caption{\textbf{Country labels retained for each Stack Exchange dataset.}
Country labels are obtained by normalizing users' free-form location strings and retaining the five most frequent countries in each dataset.}
\label{tab:country-sets}
\end{table}

\paragraph{Gender.}
Gender is inferred from users' \texttt{DisplayName}. We extract likely first names and match them to male and female name lists from the Office for National Statistics' \emph{Baby names in England and Wales: from 1996} dataset \citep{ons_babynames_2025}. We keep only entries with an unambiguous match. Gender is a widely studied demographic attribute in fairness and bias evaluation, and display names provide a common proxy signal through which such bias may enter text-based systems.

\paragraph{Reputation and engagement.}
Reputation is a platform-native status score summarizing prior community feedback, while engagement measures participation intensity based on historical interaction statistics. Although both are correlated with user experience, they may also act as socially salient status signals and induce halo effects, causing models to favor high-status or highly active users independently of answer quality.

\paragraph{Explicit vs.\ implicit settings.}
The benchmark supports both \emph{explicit} and \emph{implicit} fairness settings. In the explicit setting, the attribute label $\rv{Z}$ is directly available to the debiasing method. In the implicit setting, only the user profile text $\mathcal{H}$ is available, and $\rv{Z}$ is used solely for evaluation. Our main experiments focus on the implicit setting, which is the natural target for latent concept and attribute erasure.

\paragraph{Display-name cluster selection.}
To choose the number of display-name clusters, we compute the average silhouette score for candidate values of $k$. For each point $i$, let $a(i)$ be the average distance to points in its own cluster, and let $b(i)$ be the minimum average distance to points in any other cluster. The silhouette score is
\[
s(i)=\frac{b(i)-a(i)}{\max\{a(i),b(i)\}},
\]
and the average silhouette score for $k$ clusters is
\[
\bar{s}(k)=\frac{1}{n}\sum_{i=1}^{n}s(i).
\]
We choose
\[
k^*=\arg\max_{k\in\mathcal{K}}\bar{s}(k),
\]
which generates clusters that are compact and well separated. This procedure selects five display-name clusters in our setting.

Although protected-attribute labels are used to construct the controlled evaluation datasets, they are not provided to H-SAL when estimating the erasure subspace or training the downstream classifier.

\subsection{Evaluation Split Construction}
\label{app:split-details}

All splits in our benchmark are constructed to be user-disjoint: no \texttt{UserId} appears in more than one split. This ensures that both probe performance and fairness metrics reflect generalization to unseen users rather than recurring user-specific cues. We use an approximate 70/10/20 train--validation--test ratio by post counts, while enforcing user disjointness to prevent leakage through repeated users.

The helpfulness label $\rv{Y}$ is balanced in each split to avoid majority-class prediction. For the attribute label $\rv{Z}$, reputation, engagement, and gender are balanced where feasible. DisplayName clusters and country retain their natural distributions because they are highly imbalanced and long-tailed; fully balancing them would discard most examples and yield very small subsets.

\paragraph{Selected thresholds.}
Table~\ref{tab:selected-thresholds} reports the dataset-specific thresholds used to binarize helpfulness, reputation, and engagement. Helpfulness thresholds are chosen according to the answer-score distribution of each dataset so that the downstream task label can be balanced in each split. Reputation and engagement thresholds are selected to form non-degenerate status groups with sufficient support for user-disjoint evaluation.

\begin{table}[th!]
\centering
\begin{tabular*}{\columnwidth}{@{\extracolsep{\fill}}lrrr}
\toprule
Dataset & \makecell{Help.\\thr.} & \makecell{Rep.\\thr.} & \makecell{Eng.\\thr.} \\
\midrule
English & 3 & 185 & 6 \\
Mathematics & 2 & 466 & 22 \\
StackOverflow & 2 & 2498 & 246 \\
SuperUser & 2 & 121 & 121 \\
\bottomrule
\end{tabular*}
\caption{\textbf{Selected thresholds for binarizing labels and status attributes.}
Help., Rep., and Eng. denote helpfulness, reputation, and engagement, respectively.
Helpfulness is derived from answer score. Reputation and engagement are user-level status attributes.}
\label{tab:selected-thresholds}
\end{table}

\subsection{Example Construction}
\label{app:example-construction}
The main task is binary helpfulness prediction from Stack Exchange question--answer data. Each example is constructed at the answer level. The task input $\rv{X}$ is derived from a textual prompt formed by concatenating the question title and the answer text (e.g., ``Question: \ldots\ Answer: \ldots''). The implicit textual signal $\mathcal{H}$ is the user's profile text, taken from the \texttt{AboutMe} field. The target attribute $\rv{Z}$ is one of \textit{DisplayName}, \textit{Country}, \textit{Gender}, \textit{Reputation}, or \textit{Engagement}, depending on the experiment.

\section{Experimental Setup Details}
\subsection{Classifier and Training Details}
\label{app:classifier-details}

For downstream helpfulness prediction, we train a logistic regression classifier on top of the projected representation. We use the \texttt{scikit-learn} implementation with \texttt{max\_iter=1000} and \texttt{random\_state=0}; all other hyperparameters are left at their default values unless otherwise noted.

We use the same classifier architecture and training procedure for all experimental conditions: \textbf{Baseline} (no debiasing), \textbf{SAL (random)}, \textbf{SAL}, and \textbf{H-SAL}. Keeping the downstream classifier fixed ensures that observed differences in main-task accuracy and fairness metrics are attributable to the representation intervention itself rather than to differences in model capacity or optimization.

In the \textbf{SAL (random)} condition, we remove the same number of directions as in the corresponding debiasing setup, but select those directions at random instead of deriving them from either the explicit attribute label $\rv{Z}$ or the implicit profile representation $\rv{H}$. This condition serves as a control: if random removal produces little change while explicit or implicit erasure substantially reduces disparity, this supports the interpretation that the learned subspace is specifically aligned with the bias signal rather than reflecting an arbitrary loss of information.

For each condition, the classifier is trained on the transformed training representations and evaluated on the correspondingly transformed test representations. This design isolates the contribution of the erasure step while preserving a common prediction pipeline across all experiments.

\subsection{Choice of Erased Dimensionality}
\label{app:k-details}

A key hyperparameter in concept erasure is the number of removed directions, $k$, which controls the strength of the intervention. Smaller values of $k$ typically preserve more task-relevant information, while larger values can remove a broader bias-aligned subspace at the cost of some main-task performance. We therefore treat $k$ as a fairness--utility trade-off parameter rather than as a fixed universal setting.

In the main experiments, H-SAL and random removal use $k=50$, while explicit SAL uses its attribute-specific rank. To assess sensitivity to this choice, we additionally evaluate $k \in \{5, 10,24,45\}$ and report the resulting trends in \S\ref{sec:effect_of_k}. These results show how the strength of erasure affects both fairness and accuracy, and allow practitioners to choose an operating point appropriate for their application.

\subsection{Model Size and Compute}
\label{app:model-size-compute}

We use three backbone encoders with different model scales: BERT-base-uncased with approximately 110M parameters, Llama-3.1-8B with approximately 8B parameters, and Mistral-7B-v0.3 with approximately 7B parameters. We do not fine-tune these backbone models; they are used for representation extraction, followed by linear probing and erasure-based evaluation. Experiments were run on NVIDIA GH200 GPU nodes.

\paragraph{Cross-model transfer.}
Our experiments apply erasure separately to each backbone model and do not assume a shared erasure subspace across model families or scales. A natural direction for future work is to study whether erasure directions or bias-related mechanisms learned in one model can transfer to another model. Recent work on cross-model circuit alignment suggests that mechanistic information can sometimes be transferred from smaller to larger models through a learned alignment \citep{shao2026dfa}, although transfer becomes harder as model and architectural gaps increase.

\section{Evaluation and Analysis}

\subsection{Evaluation Metrics}
\label{app:evaluation-metrics}
To evaluate the balance between solving $\rv{Y}$ (predicting helpfulness) and debiasing $\rv{X}$, we use accuracy (for $\rv{Y}$) and TPR-Gap, defined as the RMS of class-specific group disparities \citep{DeArteaga-2019-biasbios}.

More specifically, for a specific class $c$ in the main task (\textit{Helpful} or \textit{Unhelpful}), and a protected attribute $g$ (a specific location or a reputation group), the disparity is defined as:
\begin{equation}
    \text{Gap}_c = 2 \times \sigma(\{\text{TPR}_{c, g} \mid g \in Z\}),
\end{equation}
\noindent where $\text{TPR}_{c, g} = P(\hat{\rv{Y}}=c \mid \rv{Y}=c, \rv{Z}=g)$ and $\sigma$ calculates the standard deviation. The r.v. $\hat{\rv{Y}}$ stands for the prediction, and the probability is estimated during testing based on data.\footnote{Scaling by two allows the metric to provide a unified scale for comparing both binary and multiclass attributes.}
In the case of binary $\rv{Z}$, it trivially holds $\text{Gap}_c = |\text{TPR}_0 - \text{TPR}_1|$
where $\text{TPR}_g$ is the true positive rate for
class $g \in \{ 0, 1 \}$.

For multiclass sensitive attributes (for example, five locations), we generalize the metric by calculating the root mean square (RMS) deviation of the true positive rates for each group from the global mean:
\begin{equation}
    \text{Gap}_c = 2 \times \sqrt{\frac{1}{|\{g\}|} \sum_{g} (\text{TPR}_{c,g} - \mu_c)^2},
\end{equation}
\noindent where $\mu_c$ is the average TPR across all groups for class $c$. This formulation captures the overall variation in model performance across all groups simultaneously.
Then, to obtain a single fairness scalar for the model, we aggregate the gaps calculated for every class $c$ using RMS again:
\begin{equation}
    \text{TPR-Gap} = \sqrt{ \frac{1}{|\{ c \}|} \sum_{c} (\text{Gap}_c)^2 }.
\end{equation}

This aggregation ensures that the model is penalized for disparities in both the positive and negative classes of the main task.

\subsection{Predictability of $\rv{Z}$ from $\rv{H}$: Validation of Implicit Bias}

\label{sec:validation}

As discussed in \S\ref{section:method}, our method assumes that the profile representation $\rv{H}$ contains information about the target attribute $\rv{Z}$; otherwise, there would be little attribute-related signal to erase. We validate this assumption by probing $\rv{Z}$ from $\rv{H}$ on the same user-disjoint datasets used in our main experiments. Reputation, engagement, and gender are approximately balanced to make the probing task less sensitive to majority-class prediction, while \textbf{DisplayName} clusters and country retain their natural distributions because they are highly imbalanced and long-tailed. Because the probing split is also user-disjoint, no \texttt{UserId} and corresponding \texttt{Self-description} appears in more than one split, preventing leakage of user-specific profile information between train, validation, and test sets. Table~\ref{fig:accuracy-on-self-description} shows that profile-text embeddings encode substantial attribute-related information, providing a concrete motivation for textual erasure in representation space.

\begin{table*}[!t]
\centering
\resizebox{0.88\textwidth}{!}{
\begin{tabular}{ll rrrrr}
\toprule
Dataset & Language Model & \multicolumn{4}{c}{Probing Accuracy} \\
& & Reputation & Gender & Location & Engagement & DisplayName \\
\midrule

\multirow{3}{*}{\logoso StackOverflow} & BERT & 67.4 & 74.2 & 51.0 & 67.7 & 38.1 \\
 & Llama3.1-8B & 68.0 & 74.1 & 60.6 & 67.0 & 29.9 \\
 & Mistral-7B & 67.4 & 73.2 & 58.0 & 66.9 & 26.5 \\
\midrule

\multirow{3}{*}{\logoms Mathematics} & BERT & 73.1 & 74.9 & 58.2 & 68.0 & 40.1 \\
 & Llama3.1-8B & 71.8 & 72.7 & 64.6 & 66.0 & 31.1 \\
 & Mistral-7B & 70.6 & 70.5 & 63.4 & 64.5 & 25.8 \\
\midrule

\multirow{3}{*}{\logosu SuperUser} & BERT & 64.0 & 66.0 & 48.0 & 63.6 & 32.9 \\
 & Llama3.1-8B & 63.0 & 63.2 & 56.2 & 62.1 & 25.7 \\
 & Mistral-7B & 61.7 & 66.0 & 55.3 & 62.0 & 22.7 \\
\midrule

\multirow{3}{*}{\logoes English} & BERT & 69.0 & 61.9 & 57.9 & 66.6 & 25.9 \\
 & Llama3.1-8B & 69.5 & 68.7 & 67.1 & 66.1 & 23.8 \\
 & Mistral-7B & 69.3 & 61.9 & 64.4 & 67.1 & 22.1 \\
\bottomrule

\end{tabular}
}
\caption{\textbf{Predictability of protected attributes from profile-text representations.} We report test accuracy for predicting the protected attribute $\rv{Z}$ from the representation $\rv{H}$ of users' self-descriptions. Higher accuracy indicates that the embedding retains more information about the attribute, implying stronger attribute relation to textual representations. For Location and DisplayName, there are five classes.}
\label{fig:accuracy-on-self-description}
\end{table*}

\section{Additional Results}

\label{app:additional_results}

This appendix provides additional results omitted from the main text for space.
Appendix~\ref{app:eachmodel} reports per-model results under the main erasure setting, corresponding to the model-averaged results in the main paper.
Appendix~\ref{app:different-k} reports model-averaged sensitivity to the number of removed directions \(k\), illustrating the fairness--utility trade-off as erasure strength varies.
Appendix~\ref{app:matched-k} additionally compares explicit SAL and H-SAL at matched $k=50$, while Appendix~\ref{app:bootstrap} reports bootstrap uncertainty estimates.

\subsection{Per-model Results}
\label{app:eachmodel}

The main text reports results averaged across three backbone models. Here, we provide the corresponding per-model results for BERT (Table~\ref{tab:app-bert-main}), Llama (Table~\ref{tab:app-llama-main}), and Mistral (Table~\ref{tab:app-mistral-main}) under the same user-disjoint split and main erasure setting. For implicit H-\textsc{SAL}, we use \(k=50\), which gives the strongest debiasing effect in our experiments. \textsc{SAL} uses bias-specific ranks, corresponding to the number of explicit attribute groups. \textsc{SAL} (random) uses the same \(k\) as H-\textsc{SAL}, but removes randomly selected directions as a control.

\begin{table*}[!t]
\centering
\resizebox{\textwidth}{!}{
\begin{tabular}{ll rr  rr  rr  rr}
\toprule
Bias Type & Method 
& \multicolumn{2}{c}{\textcolor{black}{\logoso \textbf{StackOverflow}}} 
& \multicolumn{2}{c}{\textcolor{black}{\logoms \textbf{Mathematics}}} 
& \multicolumn{2}{c}{\textcolor{black}{\logosu \textbf{SuperUser}}} 
& \multicolumn{2}{c}{\textcolor{black}{\logoes \textbf{English}}} \\
& & Main & TPR-Gap & Main & TPR-Gap & Main & TPR-Gap & Main & TPR-Gap \\
\midrule
\multirow{4}{*}{DisplayName} & Baseline & 68.4 & 5.9 & 63.5 & 4.9 & 74.9 & 2.9 & 68.2 & 8.3 \\
\addlinespace[4pt]
& SAL (random) & \dab{0.5} 68.0 & \da{1.8} 4.1 & \dab{0.4} 63.2 & \da{1.5} 3.4 & \dab{0.5} 74.5 & \ua{0.3} 3.2 & \dab{0.7} 67.5 & \da{0.3} 8.0 \\
\addlinespace[4pt]
& SAL & \dab{0.6} 67.8 & \da{4.3} 1.6 & \dab{1.0} 62.5 & \da{1.7} 3.2 & \dab{0.3} 74.6 & \ua{0.6} 3.5 & \uag{0.2} 68.5 & \da{0.5} 7.8 \\
\addlinespace[4pt]
& H-SAL & \dab{1.5} 66.9 & \da{5.3} 0.6 & \dab{1.3} 62.3 & \da{4.3} 0.6 & \dab{1.1} 73.8 & \da{1.4} 1.5 & \dab{2.0} 66.3 & \da{4.1} 4.1 \\
\addlinespace[4pt]
\midrule

\multirow{4}{*}{Gender} & Baseline & 55.7 & 0.5 & 57.3 & 2.9 & 58.2 & 6.1 & 53.9 & 6.7 \\
\addlinespace[4pt]
& SAL (random) & \dab{0.8} 55.0 & \ua{1.2} 1.7 & \dab{0.5} 56.8 & \da{0.6} 2.2 & \dab{3.5} 54.8 & \da{1.7} 4.4 & \dab{1.5} 52.4 & \ua{2.3} 8.9 \\
\addlinespace[4pt]
& SAL & \uag{0.4} 56.1 & \ua{0.7} 1.2 & \uag{0.4} 57.6 & \da{1.1} 1.8 & \uag{0.5} 58.8 & \da{1.0} 5.0 & \dab{0.2} 53.8 & \ua{1.2} 7.8 \\
\addlinespace[4pt]
& H-SAL & \dab{2.5} 53.3 & \ua{0.3} 0.8 & \dab{3.2} 54.1 & \da{2.0} 0.9 & \dab{2.4} 55.9 & \ua{0.7} 6.7 & \dab{2.6} 51.3 & \ua{2.8} 9.5 \\
\addlinespace[4pt]
\midrule

\multirow{4}{*}{Country} & Baseline & 66.4 & 10.7 & 60.0 & 12.8 & 63.5 & 3.6 & 61.3 & 13.0 \\
\addlinespace[4pt]
& SAL (random) & \dab{0.4} 66.0 & \da{1.3} 9.3 & \dab{0.4} 59.5 & \da{1.8} 11.1 & 63.6 & \da{0.9} 2.7 & 61.3 & \da{3.0} 9.9 \\
\addlinespace[4pt]
& SAL & \dab{0.3} 66.1 & \da{7.2} 3.5 & \dab{0.1} 59.9 & \da{1.9} 10.9 & \dab{0.3} 63.2 & \da{1.2} 2.3 & \dab{0.5} 60.8 & \da{1.1} 11.9 \\
\addlinespace[4pt]
& H-SAL & \dab{2.2} 64.2 & \da{9.2} 1.5 & \dab{4.0} 56.0 & \da{7.9} 5.0 & \dab{0.6} 62.9 & \da{0.6} 3.0 & \dab{1.6} 59.7 & \da{6.7} 6.3 \\
\addlinespace[4pt]
\midrule

\multirow{4}{*}{Reputation} & Baseline & 68.4 & 6.3 & 63.5 & 7.3 & 74.9 & 6.6 & 68.2 & 12.4 \\
\addlinespace[4pt]
& SAL (random) & \dab{0.5} 68.0 & \da{1.3} 4.9 & \dab{0.4} 63.2 & \da{0.9} 6.3 & \dab{0.5} 74.5 & \ua{0.5} 7.1 & \dab{0.7} 67.5 & \da{3.2} 9.2 \\
\addlinespace[4pt]
& SAL & \dab{0.6} 67.8 & \da{5.8} 0.4 & \dab{0.8} 62.8 & \da{6.0} 1.3 & \dab{0.8} 74.1 & \da{5.3} 1.3 & \dab{1.7} 66.5 & \da{6.2} 6.2 \\
\addlinespace[4pt]
& H-SAL & \dab{1.5} 66.9 & \da{6.2} 0.1 & \dab{1.3} 62.3 & \da{6.8} 0.5 & \dab{1.1} 73.8 & \da{4.0} 2.6 & \dab{2.0} 66.3 & \da{8.0} 4.4 \\
\addlinespace[4pt]
\midrule

\multirow{4}{*}{Engagement} & Baseline & 67.5 & 6.2 & 63.0 & 5.8 & 69.7 & 5.9 & 67.5 & 13.4 \\
\addlinespace[4pt]
& SAL (random) & \dab{0.6} 66.9 & \da{1.1} 5.1 & \dab{0.3} 62.7 & \da{1.1} 4.6 & \dab{0.1} 69.6 & \da{0.7} 5.1 & \dab{0.4} 67.1 & \da{2.1} 11.4 \\
\addlinespace[4pt]
& SAL & \dab{0.8} 66.7 & \da{5.8} 0.4 & \dab{0.5} 62.5 & \da{4.5} 1.2 & \uag{0.1} 69.8 & \da{3.7} 2.2 & \dab{0.4} 67.1 & \da{7.0} 6.5 \\
\addlinespace[4pt]
& H-SAL & \dab{1.7} 65.8 & \da{6.2} 0.0 & \dab{1.8} 61.2 & \da{5.2} 0.6 & 69.7 & \da{3.4} 2.5 & \dab{1.1} 66.4 & \da{6.1} 7.4 \\
\addlinespace[4pt]
\bottomrule
\end{tabular}
}
\caption{\textbf{BERT results under the main erasure setting.}
\textbf{Main} denotes helpfulness prediction accuracy, and \textbf{TPR-Gap} denotes generalized disparity, computed as the RMS of class-specific gaps, each defined as \(2 \times\) the standard deviation of true positive rates across groups. \textsc{SAL} uses bias-specific ranks, while \textsc{SAL} (random) and H-\textsc{SAL} use \(k=50\) removed directions for all attributes. The helpfulness label is balanced in each split; reputation, engagement, and gender are approximately balanced, while display-name clusters and country retain their natural distributions. Baseline denotes no erasure.}
\label{tab:app-bert-main}
\end{table*}

\begin{table*}[!t]
\centering
\resizebox{\textwidth}{!}{
\begin{tabular}{ll rr  rr  rr  rr}
\toprule
Bias Type & Method 
& \multicolumn{2}{c}{\textcolor{black}{\logoso \textbf{StackOverflow}}} 
& \multicolumn{2}{c}{\textcolor{black}{\logoms \textbf{Mathematics}}} 
& \multicolumn{2}{c}{\textcolor{black}{\logosu \textbf{SuperUser}}} 
& \multicolumn{2}{c}{\textcolor{black}{\logoes \textbf{English}}} \\
& & Main & TPR-Gap & Main & TPR-Gap & Main & TPR-Gap & Main & TPR-Gap \\
\midrule
\multirow{4}{*}{DisplayName} & Baseline & 69.7 & 4.5 & 64.2 & 5.0 & 69.2 & 6.0 & 61.0 & 10.2 \\
\addlinespace[4pt]
& SAL (random) & \dab{0.4} 69.3 & \da{0.2} 4.3 & \uag{0.1} 64.3 & \da{0.4} 4.6 & \uag{0.2} 69.5 & \ua{0.7} 6.7 & \dab{2.0} 59.1 & \da{1.9} 8.3 \\
\addlinespace[4pt]
& SAL & \dab{0.8} 68.8 & \da{1.9} 2.6 & \dab{1.3} 62.8 & \da{2.4} 2.6 & \dab{0.1} 69.1 & \ua{1.3} 7.3 & \dab{0.5} 60.5 & \da{2.6} 7.6 \\
\addlinespace[4pt]
& H-SAL & \dab{1.7} 68.0 & \da{2.9} 1.5 & \dab{2.2} 62.0 & \da{3.1} 1.9 & \dab{1.3} 68.0 & \da{0.6} 5.4 & \dab{1.8} 59.2 & \da{2.5} 7.7 \\
\addlinespace[4pt]
\midrule

\multirow{4}{*}{Gender} & Baseline & 56.6 & 0.0 & 54.2 & 2.2 & 58.5 & 13.5 & 54.9 & 13.8 \\
\addlinespace[4pt]
& SAL (random) & \dab{0.3} 56.3 & \ua{0.4} 0.4 & \dab{0.1} 54.1 & \ua{1.4} 3.6 & 58.5 & 13.5 & \dab{0.9} 53.9 & \ua{0.7} 14.5 \\
\addlinespace[4pt]
& SAL & \dab{0.4} 56.1 & \ua{0.5} 0.5 & \uag{0.4} 54.5 & \ua{1.5} 3.6 & 58.5 & \ua{1.0} 14.5 & \uag{0.4} 55.3 & \da{1.7} 12.1 \\
\addlinespace[4pt]
& H-SAL & \dab{1.8} 54.7 & \ua{3.4} 3.4 & \dab{1.6} 52.5 & \da{0.6} 1.6 & \dab{0.5} 58.0 & \da{0.1} 13.3 & \dab{2.8} 52.1 & \da{4.2} 9.6 \\
\addlinespace[4pt]
\midrule

\multirow{4}{*}{Country} & Baseline & 67.9 & 12.8 & 61.7 & 8.8 & 63.4 & 8.3 & 57.8 & 6.5 \\
\addlinespace[4pt]
& SAL (random) & \dab{0.2} 67.6 & \da{0.1} 12.7 & \dab{0.1} 61.5 & 8.8 & 63.3 & \da{0.2} 8.1 & \uag{0.4} 58.2 & \ua{1.2} 7.8 \\
\addlinespace[4pt]
& SAL & \dab{0.6} 67.3 & \da{10.0} 2.8 & \dab{0.1} 61.6 & \da{3.4} 5.4 & \dab{0.1} 63.3 & \da{5.7} 2.6 & \uag{1.0} 58.8 & \da{0.1} 6.5 \\
\addlinespace[4pt]
& H-SAL & \dab{2.1} 65.8 & \da{10.8} 2.0 & \dab{3.9} 57.8 & \da{4.5} 4.3 & \dab{0.8} 62.6 & \da{5.2} 3.0 & \dab{0.7} 57.1 & \ua{4.4} 10.9 \\
\addlinespace[4pt]
\midrule

\multirow{4}{*}{Reputation} & Baseline & 69.7 & 9.9 & 64.2 & 13.9 & 69.2 & 10.8 & 61.0 & 7.8 \\
\addlinespace[4pt]
& SAL (random) & \dab{0.4} 69.3 & \da{0.8} 9.1 & \uag{0.1} 64.3 & \da{0.4} 13.5 & \uag{0.2} 69.4 & 10.7 & \dab{2.0} 59.1 & \ua{0.5} 8.3 \\
\addlinespace[4pt]
& SAL & \dab{1.0} 68.6 & \da{9.7} 0.2 & \dab{1.6} 62.6 & \da{12.4} 1.5 & \dab{1.1} 68.2 & \da{4.8} 6.0 & \dab{1.0} 60.0 & \da{5.1} 2.7 \\
\addlinespace[4pt]
& H-SAL & \dab{1.7} 68.0 & \da{9.6} 0.3 & \dab{2.2} 62.0 & \da{11.5} 2.4 & \dab{1.2} 68.0 & \da{5.6} 5.2 & \dab{1.8} 59.2 & \da{4.6} 3.2 \\
\addlinespace[4pt]
\midrule

\multirow{4}{*}{Engagement} & Baseline & 69.2 & 8.4 & 63.9 & 10.1 & 65.4 & 8.9 & 60.8 & 6.7 \\
\addlinespace[4pt]
& SAL (random) & \dab{0.3} 68.9 & \da{0.7} 7.7 & \dab{0.2} 63.7 & \da{0.9} 9.2 & \dab{0.2} 65.2 & 9.0 & \dab{0.1} 60.7 & \da{1.4} 5.3 \\
\addlinespace[4pt]
& SAL & \dab{1.3} 67.9 & \da{7.9} 0.5 & \dab{1.3} 62.5 & \da{7.5} 2.6 & \dab{0.7} 64.8 & \da{5.3} 3.6 & \dab{0.9} 59.9 & \da{5.6} 1.2 \\
\addlinespace[4pt]
& H-SAL & \dab{2.2} 67.0 & \da{7.7} 0.7 & \dab{2.8} 61.0 & \da{8.0} 2.1 & \dab{0.9} 64.6 & \da{5.2} 3.7 & \dab{1.7} 59.1 & \da{1.7} 5.0 \\
\addlinespace[4pt]
\bottomrule
\end{tabular}
}
\caption{\textbf{Llama-3.1-8B results under the main erasure setting.}
Results follow the same user-disjoint evaluation protocol and method definitions as Table~\ref{tab:app-bert-main}.}
\label{tab:app-llama-main}
\end{table*}

\begin{table*}[!t]
\centering
\resizebox{\textwidth}{!}{
\begin{tabular}{ll rr  rr  rr  rr}
\toprule
Bias Type & Method 
& \multicolumn{2}{c}{\textcolor{black}{\logoso \textbf{StackOverflow}}} 
& \multicolumn{2}{c}{\textcolor{black}{\logoms \textbf{Mathematics}}} 
& \multicolumn{2}{c}{\textcolor{black}{\logosu \textbf{SuperUser}}} 
& \multicolumn{2}{c}{\textcolor{black}{\logoes \textbf{English}}} \\
& & Main & TPR-Gap & Main & TPR-Gap & Main & TPR-Gap & Main & TPR-Gap \\
\midrule
\multirow{4}{*}{DisplayName} & Baseline & 69.5 & 4.3 & 64.2 & 4.4 & 69.2 & 3.5 & 61.3 & 11.2 \\
\addlinespace[4pt]
& SAL (random) & \dab{0.2} 69.3 & \da{0.1} 4.2 & \dab{0.3} 63.9 & \ua{0.8} 5.2 & \uag{0.2} 69.4 & \da{0.7} 2.8 & \dab{0.3} 61.0 & \da{2.4} 8.8 \\
\addlinespace[4pt]
& SAL & \dab{0.8} 68.7 & \da{2.4} 1.9 & \dab{1.1} 63.1 & \da{1.8} 2.7 & \dab{0.3} 68.8 & \da{1.1} 2.3 & \dab{2.1} 59.2 & \ua{0.6} 11.8 \\
\addlinespace[4pt]
& H-SAL & \dab{1.8} 67.8 & \da{2.7} 1.6 & \dab{2.3} 62.0 & \da{2.3} 2.1 & \dab{0.6} 68.5 & \da{1.6} 1.9 & \dab{2.3} 59.0 & \da{4.3} 6.9 \\
\addlinespace[4pt]
\midrule

\multirow{4}{*}{Gender} & Baseline & 57.8 & 0.4 & 52.7 & 3.5 & 59.6 & 15.2 & 53.9 & 9.3 \\
\addlinespace[4pt]
& SAL (random) & \dab{1.2} 56.6 & \da{0.1} 0.2 & \uag{1.5} 54.2 & \da{1.0} 2.5 & \dab{1.3} 58.2 & \da{0.9} 14.4 & \dab{1.5} 52.4 & \ua{2.3} 11.6 \\
\addlinespace[4pt]
& SAL & \dab{1.6} 56.1 & \ua{0.5} 0.8 & \uag{0.4} 53.2 & \da{0.6} 2.9 & \uag{0.5} 60.1 & 15.2 & \uag{0.2} 54.1 & \ua{0.1} 9.4 \\
\addlinespace[4pt]
& H-SAL & \dab{3.0} 54.8 & 0.3 & \uag{0.6} 53.3 & \da{1.9} 1.6 & \uag{0.8} 60.4 & \da{4.0} 11.2 & 53.9 & \ua{0.3} 9.6 \\
\addlinespace[4pt]
\midrule

\multirow{4}{*}{Country} & Baseline & 67.8 & 10.9 & 61.4 & 7.5 & 63.3 & 9.1 & 58.9 & 10.6 \\
\addlinespace[4pt]
& SAL (random) & \dab{0.3} 67.5 & \da{0.6} 10.3 & \uag{0.1} 61.6 & \ua{0.4} 7.9 & \dab{0.6} 62.7 & \ua{0.8} 9.9 & \dab{0.4} 58.6 & \da{0.7} 9.9 \\
\addlinespace[4pt]
& SAL & \dab{0.9} 66.9 & \da{7.2} 3.8 & \dab{0.2} 61.3 & \da{2.7} 4.9 & \dab{0.4} 62.9 & \da{4.9} 4.2 & \dab{0.2} 58.7 & \ua{0.3} 10.9 \\
\addlinespace[4pt]
& H-SAL & \dab{2.5} 65.3 & \da{8.3} 2.6 & \dab{3.8} 57.6 & \da{4.2} 3.3 & \dab{1.3} 62.0 & \da{4.6} 4.5 & \dab{2.2} 56.7 & \ua{0.4} 11.0 \\
\addlinespace[4pt]
\midrule

\multirow{4}{*}{Reputation} & Baseline & 69.5 & 9.7 & 64.2 & 14.2 & 69.2 & 8.9 & 61.3 & 9.2 \\
\addlinespace[4pt]
& SAL (random) & \dab{0.2} 69.3 & \da{0.6} 9.1 & \dab{0.3} 63.9 & \da{0.9} 13.3 & \uag{0.2} 69.4 & 8.8 & \dab{0.3} 61.0 & \ua{0.4} 9.5 \\
\addlinespace[4pt]
& SAL & \dab{0.8} 68.7 & \da{9.6} 0.1 & \dab{1.4} 62.9 & \da{11.5} 2.7 & \dab{0.4} 68.7 & \da{4.5} 4.4 & \dab{1.2} 60.1 & \da{3.4} 5.8 \\
\addlinespace[4pt]
& H-SAL & \dab{1.8} 67.8 & \da{9.2} 0.5 & \dab{2.3} 62.0 & \da{12.0} 2.2 & \dab{0.6} 68.5 & \da{7.0} 1.8 & \dab{2.3} 59.0 & \da{2.0} 7.2 \\
\addlinespace[4pt]
\midrule

\multirow{4}{*}{Engagement} & Baseline & 69.1 & 8.5 & 64.0 & 9.5 & 66.6 & 10.5 & 62.4 & 8.6 \\
\addlinespace[4pt]
& SAL (random) & \dab{0.2} 68.8 & \da{1.2} 7.3 & \dab{0.1} 63.9 & \da{0.6} 8.9 & \dab{0.1} 66.4 & \da{1.3} 9.2 & \dab{1.2} 61.2 & \da{2.6} 6.0 \\
\addlinespace[4pt]
& SAL & \dab{1.3} 67.8 & \da{8.0} 0.5 & \dab{1.4} 62.7 & \da{7.0} 2.5 & \dab{0.6} 65.9 & \da{7.1} 3.4 & \dab{2.7} 59.7 & \da{5.4} 3.3 \\
\addlinespace[4pt]
& H-SAL & \dab{2.2} 66.9 & \da{7.8} 0.8 & \dab{2.8} 61.2 & \da{6.9} 2.6 & \dab{1.3} 65.2 & \da{6.0} 4.5 & \dab{2.8} 59.6 & \da{1.4} 7.2 \\
\addlinespace[4pt]
\bottomrule
\end{tabular}
}
\caption{\textbf{Mistral-7B results under the main erasure setting.}
Results follow the same user-disjoint evaluation protocol and method definitions as Table~\ref{tab:app-bert-main}.}
\label{tab:app-mistral-main}
\end{table*}

\subsection{Matched-$k$ Comparison with Explicit SAL}
\label{app:matched-k}

We additionally evaluate explicit SAL and H-SAL with the same number of removed directions ($k=50$). For explicit SAL, only the singular directions up to the rank of the cross-covariance matrix are identified by the explicit attribute. Removing additional directions beyond this rank is therefore not its natural operating point. We report this experiment as an additional matched-$k$ control.

\begin{table*}[!t]
\centering
\resizebox{\textwidth}{!}{
\begin{tabular}{ll rr  rr  rr  rr}
\toprule
Bias Type & Method 

& \multicolumn{2}{c}{\textcolor{black}{\logoso \textbf{StackOverflow}}} 
& \multicolumn{2}{c}{\textcolor{black}{\logoms \textbf{Mathematics}}} 
& \multicolumn{2}{c}{\textcolor{black}{\logosu \textbf{SuperUser}}} 
& \multicolumn{2}{c}{\textcolor{black}{\logoes \textbf{English}}} \\
& & Main & TPR-Gap & Main & TPR-Gap
& Main & TPR-Gap & Main & TPR-Gap \\
\midrule

DisplayName & Baseline
& 69.2 & 4.9 & 64.0 & 4.8 & 71.1 & 4.1 & 63.5 & 9.9 \\
& SAL (random)
& \dab{0.4} 68.8 & \da{0.7} 4.2
& \dab{0.2} 63.8 & \da{0.4} 4.4
& 71.1 & \ua{0.2} 4.3
& \dab{1.0} 62.5 & \da{1.5} 8.4 \\
& SAL
& \dab{1.4} 67.8 & \da{3.3} 1.6
& \dab{1.6} 62.4 & \da{2.5} 2.3
& \dab{0.6} 70.5 & \da{1.4} 2.7
& \dab{1.2} 62.3 & \da{2.8} 7.1 \\
& H-SAL
& \dab{1.7} 67.5 & \da{3.6} 1.3
& \dab{1.9} 62.1 & \da{3.2} 1.6
& \dab{1.0} 70.1 & \da{1.2} 2.9
& \dab{2.0} 61.5 & \da{3.7} 6.2 \\
\midrule
Gender & Baseline
& 56.7 & 0.3 & 54.7 & 2.8 & 58.8 & 11.6 & 54.3 & 9.9 \\
& SAL (random)
& \dab{0.8} 55.9 & \ua{0.5} 0.8
& \uag{0.3} 55.0 & 2.8
& \dab{1.6} 57.2 & \da{0.9} 10.7
& \dab{1.4} 52.9 & \ua{1.8} 11.7 \\
& SAL
& \uag{9.1} 65.8 & \ua{1.5} 1.8
& \uag{6.8} 61.5 & \ua{2.8} 5.6
& \uag{4.3} 63.1 & \da{6.8} 4.8
& \uag{2.4} 56.7 & \da{4.5} 5.4 \\
& H-SAL
& \dab{2.4} 54.3 & \ua{1.2} 1.5
& \dab{1.4} 53.3 & \da{1.4} 1.4
& \dab{0.7} 58.1 & \da{1.2} 10.4
& \dab{1.9} 52.4 & \da{0.4} 9.5 \\
\midrule
Country & Baseline
& 67.4 & 11.5 & 61.0 & 9.7 & 63.4 & 7.0 & 59.3 & 10.1 \\
& SAL (random)
& \dab{0.4} 67.0 & \da{0.7} 10.8
& \dab{0.1} 60.9 & \da{0.4} 9.3
& \dab{0.2} 63.2 & \da{0.1} 6.9
& 59.3 & \da{0.9} 9.2 \\
& SAL
& \dab{0.6} 66.8 & \da{8.1} 3.4
& 61.0 & \da{2.8} 6.9
& \dab{0.3} 63.1 & \da{3.8} 3.2
& \uag{0.3} 59.6 & \da{1.5} 8.6 \\
& H-SAL
& \dab{2.3} 65.1 & \da{9.5} 2.0
& \dab{3.9} 57.1 & \da{5.5} 4.2
& \dab{0.9} 62.5 & \da{3.5} 3.5
& \dab{1.5} 57.8 & \da{0.7} 9.4 \\
\midrule
Reputation & Baseline
& 69.2 & 8.6 & 64.0 & 11.8 & 71.1 & 8.8 & 63.5 & 9.8 \\
& SAL (random)
& \dab{0.4} 68.8 & \da{0.9} 7.7
& \dab{0.2} 63.8 & \da{0.7} 11.1
& 71.1 & \ua{0.1} 8.9
& \dab{1.0} 62.5 & \da{0.8} 9.0 \\
& SAL
& \dab{0.8} 68.4 & \da{8.3} 0.3
& \dab{1.2} 62.8 & \da{10.2} 1.6
& \dab{0.7} 70.4 & \da{5.2} 3.6
& \dab{1.4} 62.1 & \da{4.5} 5.3 \\
& H-SAL
& \dab{1.7} 67.5 & \da{8.3} 0.3
& \dab{1.9} 62.1 & \da{10.1} 1.7
& \dab{1.0} 70.1 & \da{5.6} 3.2
& \dab{2.0} 61.5 & \da{4.9} 4.9 \\
\midrule
Engagement & Baseline
& 68.6 & 7.7 & 63.6 & 8.5 & 67.2 & 8.4 & 63.6 & 9.6 \\
& SAL (random)
& \dab{0.4} 68.2 & \da{1.0} 6.7
& \dab{0.2} 63.4 & \da{0.9} 7.6
& \dab{0.1} 67.1 & \da{0.6} 7.8
& \dab{0.6} 63.0 & \da{2.0} 7.6 \\
& SAL
& \dab{1.2} 67.4 & \da{7.2} 0.5
& \dab{1.0} 62.6 & \da{6.6} 1.9
& \dab{0.4} 66.8 & \da{5.7} 2.7
& \dab{1.5} 62.1 & \da{6.0} 3.6 \\
& H-SAL
& \dab{2.0} 66.6 & \da{7.2} 0.5
& \dab{2.4} 61.2 & \da{6.7} 1.8
& \dab{0.7} 66.5 & \da{4.8} 3.6
& \dab{1.9} 61.7 & \da{3.1} 6.5 \\
\bottomrule
\end{tabular}}
\caption{Comparison with $k=50$ for all methods, averaged across \textbf{BERT}, \textbf{Llama-3.1-8B}, and \textbf{Mistral-7B}. Main and TPR-Gap are percentages. The standard explicit-SAL setting uses its attribute-specific rank.}
\label{tab:matched-k50}
\end{table*}

\subsection{Bootstrap Uncertainty}
\label{app:bootstrap}

We estimate finite-test-set uncertainty using $B=1000$ post-level bootstrap samples. For each replicate, we resample the test posts with replacement and recompute the metrics using fixed classifier predictions, without retraining. We report the original full-test results, bootstrap standard deviations, and 95\% percentile intervals.

\begin{table*}[t]
\centering
\resizebox{\textwidth}{!}{
\begin{tabular}{llcccc}
\toprule
Bias & Method
& \textcolor{black}{\logoso \textbf{StackOverflow}} Main
& \textcolor{black}{\logoso \textbf{StackOverflow}} TPR-Gap
& \textcolor{black}{\logoms \textbf{Mathematics}} Main
& \textcolor{black}{\logoms \textbf{Mathematics}} TPR-Gap \\
\midrule
DisplayName & Baseline & 0.692 (0.002; [0.687,0.697]) & 0.049 (0.005; [0.040,0.062]) & 0.640 (0.003; [0.632,0.646]) & 0.048 (0.008; [0.037,0.067]) \\
& SAL (random) & 0.688 (0.002; [0.684,0.693]) & 0.042 (0.005; [0.035,0.054]) & 0.638 (0.004; [0.631,0.645]) & 0.044 (0.007; [0.035,0.063]) \\
& SAL & 0.684 (0.002; [0.679,0.689]) & 0.020 (0.005; [0.014,0.035]) & 0.628 (0.004; [0.621,0.634]) & 0.028 (0.007; [0.022,0.047]) \\
& H-SAL & 0.675 (0.002; [0.671,0.680]) & 0.013 (0.004; [0.009,0.025]) & 0.621 (0.004; [0.613,0.627]) & 0.016 (0.005; [0.014,0.033]) \\
\midrule
Gender & Baseline & 0.567 (0.012; [0.544,0.591]) & 0.003 (0.016; [0.005,0.066]) & 0.547 (0.011; [0.529,0.569]) & 0.028 (0.017; [0.009,0.074]) \\
& SAL (random) & 0.559 (0.012; [0.537,0.582]) & 0.008 (0.015; [0.006,0.066]) & 0.550 (0.010; [0.532,0.572]) & 0.028 (0.017; [0.008,0.072]) \\
& SAL & 0.561 (0.012; [0.539,0.584]) & 0.008 (0.016; [0.006,0.064]) & 0.551 (0.010; [0.533,0.573]) & 0.028 (0.017; [0.007,0.074]) \\
& H-SAL & 0.543 (0.012; [0.520,0.568]) & 0.015 (0.017; [0.009,0.071]) & 0.533 (0.010; [0.514,0.556]) & 0.014 (0.016; [0.005,0.064]) \\
\midrule
Country & Baseline & 0.674 (0.002; [0.669,0.678]) & 0.115 (0.006; [0.105,0.127]) & 0.610 (0.003; [0.605,0.616]) & 0.097 (0.007; [0.086,0.114]) \\
& SAL (random) & 0.670 (0.002; [0.666,0.675]) & 0.108 (0.006; [0.098,0.119]) & 0.609 (0.003; [0.604,0.615]) & 0.093 (0.007; [0.081,0.109]) \\
& SAL & 0.668 (0.002; [0.663,0.673]) & 0.034 (0.005; [0.028,0.047]) & 0.609 (0.003; [0.603,0.615]) & 0.071 (0.007; [0.060,0.088]) \\
& H-SAL & 0.651 (0.003; [0.646,0.656]) & 0.020 (0.005; [0.015,0.035]) & 0.571 (0.003; [0.565,0.577]) & 0.042 (0.007; [0.034,0.060]) \\
\midrule
Reputation & Baseline & 0.692 (0.002; [0.687,0.697]) & 0.086 (0.005; [0.077,0.095]) & 0.640 (0.003; [0.632,0.646]) & 0.118 (0.007; [0.104,0.132]) \\
& SAL (random) & 0.688 (0.002; [0.684,0.693]) & 0.077 (0.005; [0.067,0.086]) & 0.638 (0.004; [0.631,0.645]) & 0.111 (0.007; [0.097,0.125]) \\
& SAL & 0.684 (0.002; [0.679,0.689]) & 0.002 (0.003; [0.001,0.012]) & 0.627 (0.004; [0.620,0.634]) & 0.018 (0.006; [0.009,0.033]) \\
& H-SAL & 0.675 (0.002; [0.671,0.680]) & 0.003 (0.003; [0.001,0.013]) & 0.621 (0.004; [0.613,0.627]) & 0.017 (0.006; [0.008,0.030]) \\
\midrule
Engagement & Baseline & 0.686 (0.002; [0.682,0.690]) & 0.077 (0.004; [0.070,0.085]) & 0.636 (0.003; [0.629,0.643]) & 0.085 (0.007; [0.072,0.100]) \\
& SAL (random) & 0.682 (0.002; [0.679,0.686]) & 0.067 (0.004; [0.060,0.075]) & 0.634 (0.003; [0.628,0.641]) & 0.076 (0.007; [0.063,0.090]) \\
& SAL & 0.675 (0.002; [0.671,0.679]) & 0.005 (0.003; [0.002,0.012]) & 0.625 (0.003; [0.619,0.632]) & 0.021 (0.007; [0.010,0.036]) \\
& H-SAL & 0.666 (0.002; [0.662,0.669]) & 0.005 (0.003; [0.001,0.013]) & 0.612 (0.004; [0.605,0.619]) & 0.018 (0.006; [0.009,0.032]) \\
\bottomrule
\end{tabular}}
\caption{Post-level bootstrap summaries ($B=1000$) for StackOverflow and Mathematics, averaged across BERT, Llama-3.1-8B, and Mistral-7B. Each entry reports Main or TPR-Gap (bootstrap SD; 95\% percentile interval). Values are on a 0--1 scale.}
\label{tab:bootstrap-so-math}
\end{table*}

\begin{table*}[t]
\centering
\resizebox{\textwidth}{!}{
\begin{tabular}{llcccc}
\toprule
Bias & Method
& \textcolor{black}{\logosu \textbf{SuperUser}} Main
& \textcolor{black}{\logosu \textbf{SuperUser}} TPR-Gap
& \textcolor{black}{\logoes \textbf{English}} Main
& \textcolor{black}{\logoes \textbf{English}} TPR-Gap \\
\midrule
DisplayName & Baseline & 0.711 (0.006; [0.700,0.723]) & 0.041 (0.014; [0.027,0.079]) & 0.635 (0.012; [0.612,0.658]) & 0.099 (0.026; [0.073,0.169]) \\
& SAL (random) & 0.711 (0.006; [0.699,0.723]) & 0.043 (0.014; [0.030,0.081]) & 0.625 (0.012; [0.602,0.648]) & 0.084 (0.023; [0.061,0.151]) \\
& SAL & 0.709 (0.006; [0.696,0.721]) & 0.044 (0.014; [0.030,0.085]) & 0.627 (0.012; [0.605,0.651]) & 0.091 (0.024; [0.067,0.158]) \\
& H-SAL & 0.701 (0.006; [0.688,0.713]) & 0.029 (0.012; [0.024,0.069]) & 0.615 (0.012; [0.592,0.637]) & 0.062 (0.022; [0.048,0.131]) \\
\midrule
Gender & Baseline & 0.588 (0.024; [0.543,0.635]) & 0.116 (0.045; [0.047,0.223]) & 0.543 (0.021; [0.499,0.582]) & 0.099 (0.040; [0.036,0.188]) \\
& SAL (random) & 0.572 (0.025; [0.528,0.622]) & 0.107 (0.046; [0.039,0.221]) & 0.529 (0.021; [0.487,0.570]) & 0.117 (0.041; [0.048,0.205]) \\
& SAL & 0.591 (0.024; [0.549,0.640]) & 0.116 (0.045; [0.055,0.223]) & 0.544 (0.021; [0.502,0.582]) & 0.098 (0.040; [0.037,0.188]) \\
& H-SAL & 0.581 (0.024; [0.534,0.629]) & 0.104 (0.046; [0.042,0.215]) & 0.524 (0.022; [0.478,0.568]) & 0.095 (0.041; [0.032,0.188]) \\
\midrule
Country & Baseline & 0.634 (0.006; [0.622,0.646]) & 0.070 (0.017; [0.048,0.112]) & 0.593 (0.008; [0.576,0.609]) & 0.101 (0.030; [0.070,0.187]) \\
& SAL (random) & 0.632 (0.006; [0.620,0.643]) & 0.069 (0.016; [0.049,0.108]) & 0.593 (0.008; [0.578,0.609]) & 0.092 (0.030; [0.064,0.182]) \\
& SAL & 0.631 (0.006; [0.620,0.643]) & 0.030 (0.015; [0.023,0.080]) & 0.595 (0.008; [0.578,0.611]) & 0.098 (0.029; [0.073,0.185]) \\
& H-SAL & 0.625 (0.006; [0.613,0.636]) & 0.035 (0.013; [0.026,0.077]) & 0.578 (0.008; [0.561,0.592]) & 0.094 (0.029; [0.068,0.182]) \\
\midrule
Reputation & Baseline & 0.711 (0.006; [0.700,0.723]) & 0.088 (0.016; [0.058,0.119]) & 0.635 (0.012; [0.612,0.658]) & 0.098 (0.024; [0.059,0.152]) \\
& SAL (random) & 0.711 (0.006; [0.699,0.723]) & 0.089 (0.016; [0.061,0.120]) & 0.625 (0.012; [0.602,0.648]) & 0.090 (0.026; [0.046,0.145]) \\
& SAL & 0.703 (0.006; [0.691,0.715]) & 0.039 (0.015; [0.014,0.068]) & 0.622 (0.012; [0.598,0.645]) & 0.049 (0.023; [0.017,0.104]) \\
& H-SAL & 0.701 (0.006; [0.688,0.713]) & 0.032 (0.014; [0.010,0.063]) & 0.615 (0.012; [0.592,0.637]) & 0.049 (0.024; [0.016,0.103]) \\
\midrule
Engagement & Baseline & 0.672 (0.007; [0.660,0.686]) & 0.084 (0.014; [0.056,0.112]) & 0.636 (0.012; [0.612,0.657]) & 0.096 (0.024; [0.055,0.147]) \\
& SAL (random) & 0.671 (0.007; [0.659,0.683]) & 0.078 (0.013; [0.052,0.104]) & 0.630 (0.012; [0.608,0.651]) & 0.076 (0.025; [0.034,0.131]) \\
& SAL & 0.668 (0.007; [0.656,0.681]) & 0.031 (0.013; [0.009,0.058]) & 0.622 (0.012; [0.597,0.644]) & 0.036 (0.021; [0.016,0.096]) \\
& H-SAL & 0.665 (0.007; [0.652,0.676]) & 0.036 (0.013; [0.015,0.065]) & 0.617 (0.012; [0.593,0.638]) & 0.065 (0.025; [0.026,0.118]) \\
\bottomrule
\end{tabular}}
\caption{Post-level bootstrap summaries ($B=1000$) for SuperUser and English, averaged across BERT, Llama-3.1-8B, and Mistral-7B. Each entry reports Main or TPR-Gap (bootstrap SD; 95\% percentile interval). Values are on a 0--1 scale.}
\label{tab:bootstrap-su-english}
\end{table*}

Bootstrap SDs and interval endpoints are averaged
across backbones and are not confidence intervals
for the cross-backbone mean.

\subsection{Sensitivity to the Number of Removed Directions}
\label{app:different-k}

The main results use \(k=50\) for H-\textsc{SAL}, which gives the strongest debiasing effect overall. To show the fairness--utility trade-off, we additionally report model-averaged results for \(k \in \{5,10,24,45\}\) in Table~\ref{tab:app-k5-average}--\ref{tab:app-k45-average}. While larger \(k\) generally reduces TPR-Gap more strongly, \(k=24\) often provides a good balance between debiasing and preserving helpfulness accuracy.

\begin{table*}[!t]
\centering
\resizebox{\textwidth}{!}{
\begin{tabular}{ll rr  rr  rr  rr}
\toprule
Bias Type & Method 
& \multicolumn{2}{c}{\textcolor{black}{\logoso \textbf{StackOverflow}}} 
& \multicolumn{2}{c}{\textcolor{black}{\logoms \textbf{Mathematics}}} 
& \multicolumn{2}{c}{\textcolor{black}{\logosu \textbf{SuperUser}}} 
& \multicolumn{2}{c}{\textcolor{black}{\logoes \textbf{English}}} \\
& & Main & TPR-Gap & Main & TPR-Gap & Main & TPR-Gap & Main & TPR-Gap \\
\midrule
\multirow{4}{*}{DisplayName} & Baseline & 69.2 & 4.9 & 64.0 & 4.8 & 71.1 & 4.1 & 63.5 & 9.9 \\
\addlinespace[4pt]
& SAL (random) & 69.2 & \da{0.1} 4.8 & 64.0 & 4.8 & \dab{0.1} 71.0 & \da{0.1} 4.1 & \dab{0.2} 63.3 & \ua{0.2} 10.1 \\
\addlinespace[4pt]
& SAL & \dab{0.8} 68.4 & \da{2.9} 2.0 & \dab{1.1} 62.8 & \da{2.0} 2.8 & \dab{0.3} 70.9 & \ua{0.3} 4.4 & \dab{0.8} 62.7 & \da{0.8} 9.1 \\
\addlinespace[4pt]
& H-SAL & \dab{1.3} 67.9 & \da{3.3} 1.6 & \dab{1.7} 62.3 & \da{3.0} 1.8 & \dab{0.5} 70.6 & \da{0.4} 3.7 & \dab{0.6} 62.9 & \da{1.5} 8.4 \\
\addlinespace[4pt]
\midrule

\multirow{4}{*}{Gender} & Baseline & 56.7 & 0.3 & 54.7 & 2.8 & 58.8 & 11.6 & 54.3 & 9.9 \\
\addlinespace[4pt]
& SAL (random) & \dab{0.6} 56.1 & \ua{1.0} 1.3 & \uag{0.5} 55.2 & \da{0.1} 2.8 & \dab{1.0} 57.8 & \ua{0.1} 11.7 & \dab{0.3} 54.0 & \ua{1.4} 11.3 \\
\addlinespace[4pt]
& SAL & \dab{0.5} 56.1 & \ua{0.5} 0.8 & \uag{0.4} 55.1 & \da{0.1} 2.8 & \uag{0.4} 59.1 & 11.6 & \uag{0.1} 54.4 & \da{0.1} 9.8 \\
\addlinespace[4pt]
& H-SAL & \dab{1.4} 55.3 & \ua{0.8} 1.1 & \dab{0.7} 54.0 & \ua{0.6} 3.5 & \uag{0.8} 59.6 & \ua{0.8} 12.4 & \dab{0.4} 53.9 & \da{1.4} 8.6 \\
\addlinespace[4pt]
\midrule

\multirow{4}{*}{Country} & Baseline & 67.4 & 11.5 & 61.0 & 9.7 & 63.4 & 7.0 & 59.3 & 10.1 \\
\addlinespace[4pt]
& SAL (random) & 67.4 & \ua{0.1} 11.5 & 61.0 & \da{0.1} 9.7 & 63.4 & \da{0.1} 6.9 & \uag{0.1} 59.4 & \ua{0.4} 10.5 \\
\addlinespace[4pt]
& SAL & \dab{0.6} 66.8 & \da{8.1} 3.4 & \dab{0.1} 60.9 & \da{2.7} 7.1 & \dab{0.3} 63.1 & \da{3.9} 3.0 & \uag{0.1} 59.5 & \da{0.3} 9.8 \\
\addlinespace[4pt]
& H-SAL & \dab{1.7} 65.7 & \da{8.3} 3.2 & \dab{2.3} 58.7 & \da{1.3} 8.4 & \dab{0.7} 62.7 & \da{3.3} 3.7 & 59.4 & \da{2.6} 7.4 \\
\addlinespace[4pt]
\midrule

\multirow{4}{*}{Reputation} & Baseline & 69.2 & 8.6 & 64.0 & 11.8 & 71.1 & 8.8 & 63.5 & 9.8 \\
\addlinespace[4pt]
& SAL (random) & 69.2 & \da{0.1} 8.5 & 64.0 & \da{0.2} 11.6 & \dab{0.1} 71.0 & \da{0.2} 8.5 & \dab{0.2} 63.3 & \ua{0.2} 9.9 \\
\addlinespace[4pt]
& SAL & \dab{0.8} 68.4 & \da{8.4} 0.2 & \dab{1.2} 62.7 & \da{10.0} 1.8 & \dab{0.8} 70.3 & \da{4.9} 3.9 & \dab{1.3} 62.2 & \da{4.9} 4.9 \\
\addlinespace[4pt]
& H-SAL & \dab{1.3} 67.9 & \da{6.5} 2.1 & \dab{1.7} 62.3 & \da{8.0} 3.8 & \dab{0.5} 70.6 & \da{2.6} 6.2 & \dab{0.6} 62.9 & \da{3.1} 6.7 \\
\addlinespace[4pt]
\midrule

\multirow{4}{*}{Engagement} & Baseline & 68.6 & 7.7 & 63.6 & 8.5 & 67.2 & 8.4 & 63.6 & 9.6 \\
\addlinespace[4pt]
& SAL (random) & \dab{0.1} 68.5 & \da{0.1} 7.6 & 63.6 & 8.5 & \dab{0.1} 67.1 & \da{0.1} 8.3 & 63.5 & \ua{0.5} 10.1 \\
\addlinespace[4pt]
& SAL & \dab{1.1} 67.5 & \da{7.2} 0.5 & \dab{1.1} 62.5 & \da{6.4} 2.1 & \dab{0.4} 66.8 & \da{5.4} 3.1 & \dab{1.3} 62.2 & \da{6.0} 3.6 \\
\addlinespace[4pt]
& H-SAL & \dab{1.5} 67.1 & \da{6.1} 1.5 & \dab{1.7} 61.9 & \da{4.9} 3.5 & \dab{0.4} 66.8 & \da{4.0} 4.4 & \dab{1.1} 62.5 & \da{2.6} 7.0 \\
\addlinespace[4pt]
\bottomrule
\end{tabular}
}
\caption{\textbf{Average results under the user-disjoint split with H-\textsc{SAL} at \(k=5\).}
Results are averaged across BERT, Llama-3.1-8B, and Mistral-7B. \textbf{Main} denotes helpfulness prediction accuracy, and \textbf{TPR-Gap} denotes generalized disparity. \textsc{SAL} uses bias-specific ranks, while \textsc{SAL} (random) and H-\textsc{SAL} use \(k=5\) removed directions for all attributes. Other evaluation details follow Table~\ref{tab:ste-summary-average-natural-disjoint-k50}.}
\label{tab:app-k5-average}
\end{table*}

\begin{table*}[th!]
\centering
\resizebox{\textwidth}{!}{
\begin{tabular}{ll rr  rr  rr  rr}
\toprule
Bias Type & Method 
& \multicolumn{2}{c}{\textcolor{black}{\logoso \textbf{StackOverflow}}} 
& \multicolumn{2}{c}{\textcolor{black}{\logoms \textbf{Mathematics}}} 
& \multicolumn{2}{c}{\textcolor{black}{\logosu \textbf{SuperUser}}} 
& \multicolumn{2}{c}{\textcolor{black}{\logoes \textbf{English}}} \\
& & Main & TPR-Gap & Main & TPR-Gap & Main & TPR-Gap & Main & TPR-Gap \\
\midrule
\multirow{4}{*}{DisplayName} & Baseline & 69.2 & 4.9 & 64.0 & 4.8 & 71.1 & 4.1 & 63.5 & 9.9 \\
\addlinespace[4pt]
& SAL (random) & \dab{0.1} 69.1 & \da{0.2} 4.7 & 64.0 & 4.8 & \dab{0.1} 71.0 & \da{0.3} 3.9 & \dab{0.5} 63.0 & \da{0.4} 9.5 \\
\addlinespace[4pt]
& SAL & \dab{0.8} 68.4 & \da{2.9} 2.0 & \dab{1.1} 62.8 & \da{2.0} 2.8 & \dab{0.3} 70.9 & \ua{0.3} 4.4 & \dab{0.8} 62.7 & \da{0.8} 9.1 \\
\addlinespace[4pt]
& H-SAL & \dab{1.4} 67.8 & \da{3.7} 1.2 & \dab{1.8} 62.1 & \da{2.9} 1.8 & \dab{0.6} 70.5 & \da{0.5} 3.6 & \dab{0.8} 62.7 & \da{3.4} 6.5 \\
\addlinespace[4pt]
\midrule

\multirow{4}{*}{Gender} & Baseline & 56.7 & 0.3 & 54.7 & 2.8 & 58.8 & 11.6 & 54.3 & 9.9 \\
\addlinespace[4pt]
& SAL (random) & \dab{0.3} 56.4 & \ua{0.8} 1.1 & \uag{0.3} 55.0 & 2.9 & \dab{1.5} 57.3 & \ua{0.1} 11.6 & \dab{0.6} 53.7 & \ua{1.7} 11.7 \\
\addlinespace[4pt]
& SAL & \dab{0.5} 56.1 & \ua{0.5} 0.8 & \uag{0.4} 55.1 & \da{0.1} 2.8 & \uag{0.4} 59.1 & 11.6 & \uag{0.1} 54.4 & \da{0.1} 9.8 \\
\addlinespace[4pt]
& H-SAL & \dab{1.9} 54.8 & \ua{1.6} 1.9 & \dab{1.0} 53.8 & \da{0.2} 2.7 & \uag{0.7} 59.5 & \da{0.9} 10.7 & \dab{0.3} 54.0 & \da{0.9} 9.0 \\
\addlinespace[4pt]
\midrule

\multirow{4}{*}{Country} & Baseline & 67.4 & 11.5 & 61.0 & 9.7 & 63.4 & 7.0 & 59.3 & 10.1 \\
\addlinespace[4pt]
& SAL (random) & 67.3 & \ua{0.1} 11.6 & 61.1 & \da{0.1} 9.7 & \dab{0.1} 63.3 & \ua{0.2} 7.1 & \dab{0.2} 59.1 & \da{1.0} 9.1 \\
\addlinespace[4pt]
& SAL & \dab{0.6} 66.8 & \da{8.1} 3.4 & \dab{0.1} 60.9 & \da{2.7} 7.1 & \dab{0.3} 63.1 & \da{3.9} 3.0 & \uag{0.1} 59.5 & \da{0.3} 9.8 \\
\addlinespace[4pt]
& H-SAL & \dab{1.9} 65.5 & \da{8.2} 3.3 & \dab{2.7} 58.3 & \da{2.7} 7.0 & \dab{0.9} 62.5 & \da{2.8} 4.2 & 59.3 & \da{1.9} 8.1 \\
\addlinespace[4pt]
\midrule

\multirow{4}{*}{Reputation} & Baseline & 69.2 & 8.6 & 64.0 & 11.8 & 71.1 & 8.8 & 63.5 & 9.8 \\
\addlinespace[4pt]
& SAL (random) & \dab{0.1} 69.1 & \da{0.1} 8.6 & 64.0 & \da{0.3} 11.5 & \dab{0.1} 71.0 & \da{0.2} 8.6 & \dab{0.5} 63.0 & \da{0.8} 9.0 \\
\addlinespace[4pt]
& SAL & \dab{0.8} 68.4 & \da{8.4} 0.2 & \dab{1.2} 62.7 & \da{10.0} 1.8 & \dab{0.8} 70.3 & \da{4.9} 3.9 & \dab{1.3} 62.2 & \da{4.9} 4.9 \\
\addlinespace[4pt]
& H-SAL & \dab{1.4} 67.8 & \da{7.2} 1.5 & \dab{1.8} 62.1 & \da{8.9} 2.9 & \dab{0.6} 70.5 & \da{3.2} 5.6 & \dab{0.8} 62.7 & \da{3.7} 6.1 \\
\addlinespace[4pt]
\midrule

\multirow{4}{*}{Engagement} & Baseline & 68.6 & 7.7 & 63.6 & 8.5 & 67.2 & 8.4 & 63.6 & 9.6 \\
\addlinespace[4pt]
& SAL (random) & \dab{0.1} 68.5 & \da{0.2} 7.5 & 63.6 & 8.4 & \dab{0.1} 67.1 & \da{0.4} 8.1 & \dab{0.2} 63.3 & \da{0.5} 9.1 \\
\addlinespace[4pt]
& SAL & \dab{1.1} 67.5 & \da{7.2} 0.5 & \dab{1.1} 62.5 & \da{6.4} 2.1 & \dab{0.4} 66.8 & \da{5.4} 3.1 & \dab{1.3} 62.2 & \da{6.0} 3.6 \\
\addlinespace[4pt]
& H-SAL & \dab{1.6} 67.0 & \da{6.3} 1.4 & \dab{1.9} 61.7 & \da{5.4} 3.1 & \dab{0.7} 66.5 & \da{3.9} 4.5 & \dab{1.2} 62.4 & \da{3.0} 6.6 \\
\addlinespace[4pt]
\bottomrule
\end{tabular}
}
\caption{\textbf{Average results under the user-disjoint split with H-\textsc{SAL} at \(k=10\).}
Results follow the same model-averaged evaluation protocol and method definitions as Table~\ref{tab:app-k5-average}.}
\label{tab:app-k10-average}
\end{table*}

\begin{table*}[th!]
\centering
\resizebox{\textwidth}{!}{
\begin{tabular}{ll rr  rr  rr  rr}
\toprule
Bias Type & Method 
& \multicolumn{2}{c}{\textcolor{black}{\logoso \textbf{StackOverflow}}} 
& \multicolumn{2}{c}{\textcolor{black}{\logoms \textbf{Mathematics}}} 
& \multicolumn{2}{c}{\textcolor{black}{\logosu \textbf{SuperUser}}} 
& \multicolumn{2}{c}{\textcolor{black}{\logoes \textbf{English}}} \\
& & Main & TPR-Gap & Main & TPR-Gap & Main & TPR-Gap & Main & TPR-Gap \\
\midrule
\multirow{4}{*}{DisplayName} & Baseline & 69.2 & 4.9 & 64.0 & 4.8 & 71.1 & 4.1 & 63.5 & 9.9 \\
\addlinespace[4pt]
& SAL (random) & \dab{0.3} 68.9 & \da{0.6} 4.3 & 64.0 & \da{0.2} 4.6 & \dab{0.1} 71.0 & \da{0.2} 3.9 & \dab{0.6} 63.0 & \da{1.2} 8.7 \\
\addlinespace[4pt]
& SAL & \dab{0.8} 68.4 & \da{2.9} 2.0 & \dab{1.1} 62.8 & \da{2.0} 2.8 & \dab{0.3} 70.9 & \ua{0.3} 4.4 & \dab{0.8} 62.7 & \da{0.8} 9.1 \\
\addlinespace[4pt]
& H-SAL & \dab{1.5} 67.7 & \da{3.8} 1.1 & \dab{2.0} 62.0 & \da{3.2} 1.6 & \dab{1.0} 70.1 & \da{0.5} 3.6 & \dab{1.1} 62.4 & \da{3.2} 6.7 \\
\addlinespace[4pt]
\midrule

\multirow{4}{*}{Gender} & Baseline & 56.7 & 0.3 & 54.7 & 2.8 & 58.8 & 11.6 & 54.3 & 9.9 \\
\addlinespace[4pt]
& SAL (random) & \dab{0.4} 56.3 & \ua{0.8} 1.1 & \uag{0.4} 55.1 & 2.9 & \dab{0.6} 58.2 & \ua{0.9} 12.4 & \dab{0.6} 53.6 & \ua{2.0} 11.9 \\
\addlinespace[4pt]
& SAL & \dab{0.5} 56.1 & \ua{0.5} 0.8 & \uag{0.4} 55.1 & \da{0.1} 2.8 & \uag{0.4} 59.1 & 11.6 & \uag{0.1} 54.4 & \da{0.1} 9.8 \\
\addlinespace[4pt]
& H-SAL & \dab{1.8} 54.9 & \ua{1.1} 1.4 & \dab{0.8} 53.9 & \da{1.7} 1.2 & \dab{1.0} 57.8 & \ua{0.2} 11.8 & \dab{0.7} 53.6 & \da{1.7} 8.2 \\
\addlinespace[4pt]
\midrule

\multirow{4}{*}{Country} & Baseline & 67.4 & 11.5 & 61.0 & 9.7 & 63.4 & 7.0 & 59.3 & 10.1 \\
\addlinespace[4pt]
& SAL (random) & \dab{0.2} 67.1 & \da{0.3} 11.1 & \dab{0.1} 60.9 & \da{0.4} 9.3 & \dab{0.1} 63.3 & \da{0.2} 6.7 & 59.4 & \da{1.8} 8.3 \\
\addlinespace[4pt]
& SAL & \dab{0.6} 66.8 & \da{8.1} 3.4 & \dab{0.1} 60.9 & \da{2.7} 7.1 & \dab{0.3} 63.1 & \da{3.9} 3.0 & \uag{0.1} 59.5 & \da{0.3} 9.8 \\
\addlinespace[4pt]
& H-SAL & \dab{2.0} 65.4 & \da{9.1} 2.4 & \dab{3.1} 57.9 & \da{5.6} 4.1 & \dab{0.9} 62.5 & \da{3.0} 4.0 & \dab{0.9} 58.4 & \da{2.6} 7.4 \\
\addlinespace[4pt]
\midrule

\multirow{4}{*}{Reputation} & Baseline & 69.2 & 8.6 & 64.0 & 11.8 & 71.1 & 8.8 & 63.5 & 9.8 \\
\addlinespace[4pt]
& SAL (random) & \dab{0.3} 68.9 & \da{0.6} 8.0 & 64.0 & \da{0.5} 11.4 & \dab{0.1} 71.0 & 8.7 & \dab{0.6} 63.0 & \da{0.6} 9.2 \\
\addlinespace[4pt]
& SAL & \dab{0.8} 68.4 & \da{8.4} 0.2 & \dab{1.2} 62.7 & \da{10.0} 1.8 & \dab{0.8} 70.3 & \da{4.9} 3.9 & \dab{1.3} 62.2 & \da{4.9} 4.9 \\
\addlinespace[4pt]
& H-SAL & \dab{1.5} 67.7 & \da{7.9} 0.7 & \dab{2.0} 62.0 & \da{9.7} 2.1 & \dab{0.9} 70.2 & \da{4.4} 4.3 & \dab{1.1} 62.4 & \da{4.4} 5.3 \\
\addlinespace[4pt]
\midrule

\multirow{4}{*}{Engagement} & Baseline & 68.6 & 7.7 & 63.6 & 8.5 & 67.2 & 8.4 & 63.6 & 9.6 \\
\addlinespace[4pt]
& SAL (random) & \dab{0.2} 68.4 & \da{0.7} 7.0 & \dab{0.1} 63.5 & \da{0.5} 8.0 & 67.2 & \da{0.3} 8.1 & \dab{0.3} 63.3 & \da{0.3} 9.3 \\
\addlinespace[4pt]
& SAL & \dab{1.1} 67.5 & \da{7.2} 0.5 & \dab{1.1} 62.5 & \da{6.4} 2.1 & \dab{0.4} 66.8 & \da{5.4} 3.1 & \dab{1.3} 62.2 & \da{6.0} 3.6 \\
\addlinespace[4pt]
& H-SAL & \dab{1.8} 66.8 & \da{7.1} 0.6 & \dab{2.3} 61.3 & \da{6.3} 2.2 & \dab{0.7} 66.5 & \da{4.5} 3.9 & \dab{1.4} 62.2 & \da{2.6} 7.0 \\
\addlinespace[4pt]
\bottomrule
\end{tabular}
}
\caption{\textbf{Average results under the user-disjoint split with H-\textsc{SAL} at \(k=24\).}
Results follow the same model-averaged evaluation protocol and method definitions as Table~\ref{tab:app-k5-average}.}
\label{tab:app-k24-average}
\end{table*}

\begin{table*}[th!]
\centering
\resizebox{\textwidth}{!}{
\begin{tabular}{ll rr  rr  rr  rr}
\toprule
Bias Type & Method 
& \multicolumn{2}{c}{\textcolor{black}{\logoso \textbf{StackOverflow}}} 
& \multicolumn{2}{c}{\textcolor{black}{\logoms \textbf{Mathematics}}} 
& \multicolumn{2}{c}{\textcolor{black}{\logosu \textbf{SuperUser}}} 
& \multicolumn{2}{c}{\textcolor{black}{\logoes \textbf{English}}} \\
& & Main & TPR-Gap & Main & TPR-Gap & Main & TPR-Gap & Main & TPR-Gap \\
\midrule
\multirow{4}{*}{DisplayName} & Baseline & 69.2 & 4.9 & 64.0 & 4.8 & 71.1 & 4.1 & 63.5 & 9.9 \\
\addlinespace[4pt]
& SAL (random) & \dab{0.4} 68.8 & \da{0.7} 4.1 & 64.0 & \da{0.7} 4.1 & 71.1 & \ua{0.3} 4.5 & \dab{0.7} 62.8 & \da{1.7} 8.2 \\
\addlinespace[4pt]
& SAL & \dab{0.8} 68.4 & \da{2.9} 2.0 & \dab{1.1} 62.8 & \da{2.0} 2.8 & \dab{0.3} 70.9 & \ua{0.3} 4.4 & \dab{0.8} 62.7 & \da{0.8} 9.1 \\
\addlinespace[4pt]
& H-SAL & \dab{1.6} 67.6 & \da{3.5} 1.4 & \dab{1.8} 62.1 & \da{3.2} 1.6 & \dab{1.1} 70.1 & \da{0.7} 3.4 & \dab{1.8} 61.7 & \da{4.3} 5.6 \\
\addlinespace[4pt]
\midrule

\multirow{4}{*}{Gender} & Baseline & 56.7 & 0.3 & 54.7 & 2.8 & 58.8 & 11.6 & 54.3 & 9.9 \\
\addlinespace[4pt]
& SAL (random) & \dab{0.9} 55.8 & \ua{0.9} 1.2 & \uag{0.3} 55.0 & \ua{0.1} 2.9 & \dab{0.9} 57.9 & \da{2.0} 9.6 & \dab{0.9} 53.4 & \ua{2.5} 12.4 \\
\addlinespace[4pt]
& SAL & \dab{0.5} 56.1 & \ua{0.5} 0.8 & \uag{0.4} 55.1 & \da{0.1} 2.8 & \uag{0.4} 59.1 & 11.6 & \uag{0.1} 54.4 & \da{0.1} 9.8 \\
\addlinespace[4pt]
& H-SAL & \dab{2.3} 54.4 & \ua{1.1} 1.4 & \dab{1.3} 53.4 & \da{1.2} 1.6 & \dab{2.0} 56.8 & \da{0.5} 11.1 & \dab{1.1} 53.1 & 9.9 \\
\addlinespace[4pt]
\midrule

\multirow{4}{*}{Country} & Baseline & 67.4 & 11.5 & 61.0 & 9.7 & 63.4 & 7.0 & 59.3 & 10.1 \\
\addlinespace[4pt]
& SAL (random) & \dab{0.4} 67.0 & \da{0.8} 10.7 & \dab{0.1} 60.9 & \da{0.4} 9.3 & \dab{0.2} 63.2 & \da{0.2} 6.8 & 59.4 & \da{1.2} 8.8 \\
\addlinespace[4pt]
& SAL & \dab{0.6} 66.8 & \da{8.1} 3.4 & \dab{0.1} 60.9 & \da{2.7} 7.1 & \dab{0.3} 63.1 & \da{3.9} 3.0 & \uag{0.1} 59.5 & \da{0.3} 9.8 \\
\addlinespace[4pt]
& H-SAL & \dab{2.3} 65.1 & \da{9.3} 2.2 & \dab{3.7} 57.3 & \da{5.9} 3.8 & \dab{0.8} 62.6 & \da{3.1} 3.9 & \dab{1.5} 57.8 & \da{1.2} 8.9 \\
\addlinespace[4pt]
\midrule

\multirow{4}{*}{Reputation} & Baseline & 69.2 & 8.6 & 64.0 & 11.8 & 71.1 & 8.8 & 63.5 & 9.8 \\
\addlinespace[4pt]
& SAL (random) & \dab{0.4} 68.8 & \da{0.8} 7.8 & 64.0 & \da{0.7} 11.1 & 71.1 & \ua{0.2} 8.9 & \dab{0.7} 62.8 & \da{1.4} 8.3 \\
\addlinespace[4pt]
& SAL & \dab{0.8} 68.4 & \da{8.4} 0.2 & \dab{1.2} 62.7 & \da{10.0} 1.8 & \dab{0.8} 70.3 & \da{4.9} 3.9 & \dab{1.3} 62.2 & \da{4.9} 4.9 \\
\addlinespace[4pt]
& H-SAL & \dab{1.6} 67.6 & \da{8.3} 0.3 & \dab{1.8} 62.1 & \da{9.9} 2.0 & \dab{1.1} 70.1 & \da{5.1} 3.7 & \dab{1.8} 61.7 & \da{4.8} 5.0 \\
\addlinespace[4pt]
\midrule

\multirow{4}{*}{Engagement} & Baseline & 68.6 & 7.7 & 63.6 & 8.5 & 67.2 & 8.4 & 63.6 & 9.6 \\
\addlinespace[4pt]
& SAL (random) & \dab{0.4} 68.2 & \da{1.1} 6.6 & \dab{0.2} 63.5 & \da{0.8} 7.7 & \dab{0.2} 67.0 & \da{0.3} 8.1 & \dab{0.8} 62.8 & \da{0.8} 8.8 \\
\addlinespace[4pt]
& SAL & \dab{1.1} 67.5 & \da{7.2} 0.5 & \dab{1.1} 62.5 & \da{6.4} 2.1 & \dab{0.4} 66.8 & \da{5.4} 3.1 & \dab{1.3} 62.2 & \da{6.0} 3.6 \\
\addlinespace[4pt]
& H-SAL & \dab{2.0} 66.6 & \da{7.2} 0.5 & \dab{2.5} 61.1 & \da{6.6} 1.8 & \dab{0.8} 66.4 & \da{4.8} 3.6 & \dab{1.8} 61.8 & \da{2.4} 7.2 \\
\addlinespace[4pt]
\bottomrule
\end{tabular}
}
\caption{\textbf{Average results under the user-disjoint split with H-\textsc{SAL} at \(k=45\).}
Results follow the same model-averaged evaluation protocol and method definitions as Table~\ref{tab:app-k5-average}.}
\label{tab:app-k45-average}
\end{table*}

\end{document}